%% file: main.tex
\documentclass[twoside]{article}

\usepackage[accepted]{aistats2024}

\usepackage{titlesec}

\usepackage[ruled]{algorithm2e}
\usepackage{float}
\usepackage{lipsum}
\usepackage{natbib}
\usepackage[T1]{fontenc}    
\usepackage{url}            
\usepackage{booktabs}       
\usepackage{amsfonts}       
\usepackage{nicefrac}       
\usepackage{microtype}      
\usepackage{color}
\usepackage{amssymb}
\usepackage{amsmath}
\usepackage{multirow,multicol}
\usepackage{graphicx}
\usepackage{array}

\usepackage{bm}
\usepackage{amsthm}
\newcommand{\E}{\mathbb{E}}
\usepackage{mathtools}
\usepackage{mathstyle}       

\newtheorem{lemma}{Lemma}

\newtheorem{assumption}{Assumption}
\newtheorem{remark}{Remark}
\newtheorem{theorem}{Theorem}

\makeatletter
\newcommand{\removelatexerror}{\let\@latex@error\@gobble}
\makeatother

\title{RL in Markov Games with Independent  Function Approximation: Improved Sample Complexity Bound under the Local Access Model}

\begin{document}
%
\runningtitle{RL in Markov Games with Independent Function Approximation}

%
\runningauthor{Junyi Fan, Yuxuan Han, Jialin Zeng, Jian-Feng Cai, Yang Wang, Yang Xiang, Jiheng Zhang}

\twocolumn[

\aistatstitle{RL in Markov Games with Independent Function Approximation: Improved Sample Complexity Bound under the Local Access Model}

\aistatsauthor{  Junyi Fan$^{*\dagger}$ \And Yuxuan Han$^{*\dagger}$ \And  Jialin Zeng$^{*\dagger}$ \And Jian-Feng Cai$^{\dagger \S}$ }
\aistatsauthor{  Yang Wang$^{\dagger \ddagger }$ \And Yang Xiang$^{\dagger \S}$ \And Jiheng Zhang$^{\dagger \ddagger }$ }

\aistatsaddress{$\dagger$ Department of Mathematics, HKUST \\   $\ddagger$Department of Industrial Engineering and Decision Analytics, HKUST \\  $\S$HKUST Shenzhen-Hong Kong Collaborative Innovation Research Institute\\
$*$Equal Contribution, Correspondence to: maxiang@ust.hk, jiheng@ust.hk}

]

\begin{abstract}
Efficiently learning equilibria with large state and action spaces in general-sum Markov games while overcoming the curse of multi-agency is a challenging problem.
Recent works have attempted to solve this problem by employing independent linear function classes to approximate the marginal $Q$-value for each agent.
However, existing sample complexity bounds under such a framework have a suboptimal dependency on the desired accuracy $\varepsilon$ or the action space.
In this work, we introduce a new algorithm, Lin-Confident-FTRL, for learning coarse correlated equilibria (CCE) with local access to the simulator,  i.e., one can interact with the underlying environment on the visited states.
 Up to 
 a logarithmic dependence on the size of the state space, Lin-Confident-FTRL learns $\epsilon$-CCE with a provable optimal accuracy bound $O(\epsilon^{-2})$
and gets rids of the linear dependency on the action space, while scaling polynomially with relevant problem parameters (such as the number of agents and time horizon).
Moreover, our analysis of Linear-Confident-FTRL generalizes the virtual policy iteration technique in the single-agent local planning literature,
which yields a new computationally efficient algorithm with a tighter sample complexity bound when assuming random access to the simulator.
\end{abstract}

\input{introduction.tex}

\input{preliminary.tex}

%
\input{linear_game.tex}
\bibliographystyle{apalike}
\bibliography{main}

\onecolumn

\newpage

\appendix

\input{appendix_policy_class}

\input{appendix_computation}

\input{appendix_local}

\input{appendix_generative}

\input{appendix_aux_lem}

\end{document}

%% file: introduction.tex
\section{Introduction}

As a flourishing subfield of reinforcement learning, multi-agent reinforcement learning (MARL) systems have demonstrated impressive success across a variety of modern artificial intelligence tasks, such as chess and GO games \citep{silver2017mastering}, Poker \citep{brown2019superhuman}, autonomous self-driving \citep{shalev2016safe}, and multi-robot controls\citep{matignon2012coordinated}. MARL investigates how multiple agents interact in an unknown shared environment and learn to take actions that maximize their individual reward. 
Compared to single-agent RL, where an agent only needs to optimize its own behavior by interacting with the environment, the presence of complex interactions among multiple players in MARL poses 
 some novel challenges.

MARL encounters similar challenges as single-agent reinforcement learning in dealing with large state and action spaces, which are further compounded in the multi-agent scenario. 
In single-agent RL, function approximation is widely employed to tackle the challenges arising from large state and action spaces that cannot be exhaustively explored \citep{CisnerosVelarde2023FinitesampleGF,wen2017efficient, jiang2017contextual, du2019good, yang2020reinforcement, jin2020provably, wang2020reinforcement, zanette2020learning,jin2021bellman,du2021bilinear,yin2022efficient,foster2021statistical}. 
However, applying function approximation to MARL using a global function approximation that captures the joint $Q$-value of all agents results in the curse of multi-agency, where the sample complexity scales exponentially with the number of agents \citep{xie2020learning, huang2021towards, chen2021almost,jin2022power, chen2022unified,ni2022representation}.
To address this problem, decentralized, or independent linear function approximation has been proposed in \cite{wang2023breaking, cui2023breaking} for learning equilibrium in multi-agent general-sum Markov games, where the linear function class only models the marginal $Q$-value for each agent.   Specifically, \cite{wang2023breaking} combine new policy replay mechanisms with $V$-learning that can learn $\varepsilon$-coarse correlated equilibirum (CCE) with $O(\varepsilon^{-2})$ sample complexity. However, the sample complexity of their algorithm still depends polynomially on the size of the largest action space
$\text{max}_{i=1}^m A_i$, which is affected by the large action space issue and does not fully utilize the advantage of function approximation.
 \cite{cui2023breaking} also employ policy replay techniques but with on-policy samples that eliminate the dependency on the number of actions. However, this approach yields a sub-optimal sample complexity of  $O(\varepsilon^{-4})$ for finding $\varepsilon$-CCE.

\input{table.tex}

  While both \cite{wang2023breaking} and \cite{cui2023breaking} have utilized the online access model, it is reasonable to believe that compared to the online access model, more flexible sampling protocols, such as local access or random access models can lead to an improved sample complexity.  This observation raises the following open question:
 \begin{center}
 \textit{Can we design more sample-efficient algorithms for MARL with independent linear function approximation under stronger access models?}
 \end{center}
  In this paper, we make an effort to answer this question by designing an algorithm that achieves sharper dependency 
  under the local access model and random access model.  Random access model, also known as generative model, allows the player to query any state-action pair. Recently, the local access model has gained popularity in the single-agent RL with function approximation both theoretically \citep{weisz2022confident, yin2022efficient,hao2022confident,li2021sample} and empirically \citep{Tavakoli2020ExploringRD,lan2023can,Yin2023SampleED}. This model allows the agent to query the simulator with previously visited states, providing more versatility than the random access model and accommodating many realistic scenarios. For example, in many video games, players can revisit previously recorded states. We summarize our key contributions and technical innovations under these two models below.

\subsection{Our Contribution}

\textbf{Independent linear Markov Games under the local access model. }
We propose a more efficient algorithm, Linear-Confident-FTRL, for independent linear Markov games with local access to a simulator.
To leverage accumulated information and prevent unnecessary revisits, the algorithm maintains a distinct core set of state-action pairs for each agent, which then determine a common confident state set. Then each agent performs policy learning over his own core set. Whenever a new state outside the confident state set is detected during the learning, the core set is expanded, and policy learning is restarted for all agents. To conduct policy learning, the algorithm employs a decentralized Follow-The-Regularized-Leader (FTRL) subroutine, which is executed by each agent over their own core sets, utilizing an adaptive sampling strategy extended from the tabular and the random model setting \citep{li2022minimax}. 
This adaptive sampling strategy effectively mitigates the curse of multi-agency, which is caused by uniform sampling over all state-action pairs.


\textbf{Sample complexity bound under the local access model. } By querying from the local access model, the Linear-Confident-FTRL algorithm is provable to learn an $\varepsilon$-CCE with {\small $\tilde{O}(\min\{\frac{\log(S)}{d},\max_i A_i\}d^3H^6m^2\varepsilon^{-2})$} samples for independent linear Markov Games. Here, $d$ denotes the dimension of the linear function, $S$ is the size of the state space, $m$ represents the number of agents, $H$ stands for the time horizon, and $A_i$ is the number of actions for player $i$.
When $S \lesssim e^{d\max_iA_i}$,  we get rid of the dependency on action space and achieve near-optimal dependency on $\varepsilon$.
For possibly infinite $S$, our algorithm achieves $\tilde{O}(\varepsilon^{-2}d^3H^6m^2\max_i A_i)$ sample complexity, which is similar to \cite{wang2023breaking} but sharpens the dependency on $\max_i{A}_i$ and $d$.
We make detailed comparisons with prior works in table~\ref{table:MG}.

\textbf{Sample complexity bound under the random access model. } 
Our analysis of Linear-Confident-FTRL generalizes the virtual policy iteration technique in the single-agent local planning literature \citep{hao2022confident,yin2022efficient}, in which a virtual algorithm is constructed and used as a bridge to analyze the performance of the main algorithm. In particular,
our construction of the  virtual algorithm also yields a new algorithm with a tighter sample complexity bound $\tilde{O}(\min\{\varepsilon^{-2}dH^2,\frac{\log(S)}{d},\max_i A_i\}d^2H^6m^2\varepsilon^{-2})$ when the random access to the simulator is available. 
It is worth noting that the minimax lower bound in the tabular case is $\Omega(S\max_i A_iH^4\varepsilon^{-2})$ \citep{li2022minimax}. Since the independent linear approximation recovers the tabular case with $d = S\max_i A_i$, a lower bound of $\Omega(dH^4\varepsilon^{-2})$ can be derived within this framework. By comparing our sample complexity bound to this lower bound, we can demonstrate that when $S$ is not exponentially large, our proposed algorithm under the random access model achieves optimal dependency on $d$ and $\varepsilon$.
On the  other hand, for possibly infinity $S$, our sample complexity bound achieves the minimum over the $\tilde{O}(\varepsilon^{-4})$ result in \cite{cui2023breaking} and the $\tilde{O}(\varepsilon^{-2}A)$ result in \cite{wang2023breaking}, with all other problem-relevant parameters are sharpened.

\subsection{Related Work}

\textbf{Multi-Agent Markov Game. }
There exist plenty of prior works in Multi-Agent Games, which offer wide exploration of different algorithms under different settings. \cite{zhang2020model, liu2021sharp} provide model-based algorithms under different sampling protocols, while exponential growth on the number of agents ($\Pi_{i\in [m]}A_i$) are induced in the sample complexity. \cite{bai2020nearoptimal, song2022learn,jin2021v, mao2022improving} circumvent the curse of multi-agency via decentralized algorithms but return the non-Markov policies. \cite{daskalakis2022complexity} propose an algorithm producing Markov policies, which only depend on current state information, but at the cost of higher sample complexity. In the tabular multi-agent game, \cite{li2022minimax} provide the first algorithm for learning the $\varepsilon$-NE in two players zero-sum game and $\varepsilon$-CCE in multi-player general-sum game with minimax optimal sample complexity bound under the random access model.

\textbf{Function Approximation in RL. }
The function approximation framework has been widely applied in single-agent RL with large state and action spaces \citep{zanette2020learning, jin2020provably, yang2020reinforcement,jin2021bellman, wang2020reinforcement,du2021bilinear,foster2021statistical}. The same framework has also been generalized to Markov games \citep{xie2020learning, chen2021almost,jin2022power, huang2021towards,ni2022representation} in a centralized manner, i.e., they approximate the joint $Q$ function defined on $\mathcal{S} \times \prod_{i\in [m]}\mathcal{A}_i$, which results in the complexity of the considered function class inherently depend on $\prod_{i\in[m]}A_i$. In contrast, we consider the function approximation in a decentralized manner as in  \cite{cui2023breaking,wang2023breaking} to get rid of the curse of the multi-agency.


\textbf{RL under Local Access Model. }
Single-agent RL with linear function approximation under the local access model has been well investigated in previous works \citep{li2021sample,wang2021exponential,weisz2021query,yin2022efficient,hao2022confident,weisz2022confident}. \cite{yin2022efficient,hao2022confident} propose provably efficient algorithms for single-agent learning under the linear realizability assumption.  
Their algorithm design and analysis rely on the concept of the core set and the construction of virtual algorithms, which we have generalized in our paper to the multi-agent setting using a decentralized approach.
The only work considering multi-agent learning under local access, to our knowledge, is \cite{tkachuk2023efficient}. They consider the \textit{cooperative} multi-agent learning but with \textit{global} linear function approximation. They focus on designing sample efficient algorithms for learning the globally optimal policy with computational complexity scales in $\text{Poly}(\max_i{A}_i,d)$ instead of $\text{Poly}(\Pi_i{A}_i,d)$ under the additive decomposition assumption on the global $Q$ function.
In contrast, our work addresses the learning CCE of \textit{general-sum} Markov games with \textit{independent} function approximation in a \textit{decentralized} manner.
It's important to highlight that while the general-sum game encompasses the cooperative game as a particular instance, the CCE policy might not always align with the global optimal policy. 
This distinction complicates a direct comparison between our results and those presented in \cite{tkachuk2023efficient}.
Lastly, we would like to point out that the computational complexity of our algorithm  also scales in $\text{Poly}(\max_i A_i,d)$.  This is directly inferred from our algorithm design detailed in Section~3.

\paragraph{Recent Refined Sample Complexity Bounds under Online Access}
After the submission of our paper, a recent and independent work by \citet{dai2024refined} studied the same problem in the online access setting and achieved significant improvements in the sample complexity bounds originally presented by \citet{cui2023breaking} and \citet{wang2023breaking}. By utilizing tools developed for the single-agent setting in \citet{dai2023refined} and refining the AVLPR scheme of \citet{wang2023breaking}, \citet{dai2024refined} demonstrated that it is possible to obtain a sample complexity bound of the order $\tilde{O}(\frac{m^4d^5H^6\log S}{\varepsilon^2})$. Notably, \cite{dai2024refined} achieved a similar sample complexity bound as our results under a weaker access model than ours, where the dependency on $\varepsilon$ is optimal, no polynomial dependency in $A$ is incurred, and only logarithmic dependency on $S$ is present. Refining both our results and those of \citet{dai2024refined} to achieve bounds completely independent of the state space size $S$, while maintaining favorable dependencies on $A$ and $\varepsilon$, remains a challenging task and is left as a valuable direction for future research.

%% file: table.tex
\begingroup 
\renewcommand{\arraystretch}{1.4} 

\begin{table*}[ht!]\label{table:MG}
    \centering
    \resizebox{\textwidth}{!}{
    \renewcommand{\arraystretch}{1.5}
    \begin{tabular}{|c|c|c|c|}
    \hline
    \textbf{Result} & \textbf{Sample Complexity} & \textbf{Tabular Case Complexity} &\textbf{Sampling Protocol}  \\ \hline

 Theorem~6, \cite{jin2021v} & $\tilde{O}({ H^6S\max_iA_i}{\varepsilon^{-2}})$ & \multirow{4}{*}{N.A.} & Online Access \\         \cline{1-2} \cline{4-4}

   Theorem~3.3, \cite{zhang2020generative} & $\tilde{O}({ H^3SA_1A_2}{\varepsilon^{-2}})$ &  & Random Access \\ 
 \cline{1-2} \cline{4-4}    
   Theorem~2, \cite{li2022minimax} & $\tilde{O}({ H^4S\sum_{i=1}^m A_i}{\varepsilon^{-2}})$ &  & Random Access \\ 
    \hline

    Theorem~2, \cite{xie2020learning} & $\tilde{O}({d^3H^4}{\varepsilon^{-2}})$ &\multirow{1}{*}{$d = SA_1A_2$} & Online Access \\ 
    \hline

     Theorem~5.2, \cite{chen2021almost} & $\tilde{O}({d^2H^3}{\varepsilon^{-2}})$ &\multirow{1}{*}{$d = S^2A_1A_2$} & Online Access \\ 
    \hline

     Theorem~5, \cite{wang2023breaking} & $\tilde{O}({d^4H^6m^2\max_i A_i^5}{\varepsilon^{-2}})$ & \multirow{4}{*}{$d = S\max_i A_i$} & Online Access \\ 
    \cline{1-2} \cline{4-4}
    
     Theorem~1, \cite{cui2023breaking} & $\tilde{O}(d^4H^{10}m^4\varepsilon^{-4})$ & & Online Access \\\cline{1-2} \cline{4-4}

    Theorem~4, \citet{dai2024refined}$^\ddagger$& $\tilde{O}(m^4d^5H^6 \log S \varepsilon^{-2} )$ & & Online Access \\\cline{1-2} \cline{4-4}

    Theorem~1 \textbf{(This Paper)}$^\dagger$& $\tilde{O}(\min\{
\frac{\log(S)}{d}, \max_i A_i\}d^3H^6m^2\varepsilon^{-2})$ & & Local Access \\
    \cline{1-2} \cline{4-4}

    Theorem~2 \textbf{(This Paper)}   & $\tilde{O}(\min\{\varepsilon^{-2}dH^2,\frac{\log(S)}{d}, \max_i A_i\}d^2H^5m\varepsilon^{-2})$ & & Random Access \\ \hline
\end{tabular}}

\caption{
Comparison of different algorithms, where in $\tilde{O}(\cdot)$ we omit $\text{polylog}(A,H,m,d,\varepsilon)$ terms. Results in \cite{zhang2020generative,chen2021almost,xie2020learning} are for learning the $\varepsilon$-Nash Equilibrium(NE) in two player zero-sum Markov Games while other results are for learning $\varepsilon$-CCE for $m$-player general-sum Markov Games.\\
\footnotesize{$^\ddagger$ See the last paragraph of section~1.2.}\\
\footnotesize{$^\dagger$ When $\min\{d^{-1}\log S, \max_iA_i \} \geq \varepsilon^{-2}$, Theorem~1 also ensures a sample complexity bound independent of $\log S$ and $A$, details are presented in section~3.}
}
\end{table*}
\endgroup

%% file: preliminary.tex
\section{Preliminaries}
\paragraph{Notation} For a positive integer $m$, we use $[m]$ to denote $\{1,\dots,m\}.$ We write $a\lesssim b$ or $a = \tilde{O}(b)$ to denote $a\leq C\text{polylog}\big(A,m,\varepsilon^{-1},H, \log(1/\delta)\big) \cdot b$ for some absolute constant $C$.  We use $\lVert \cdot \rVert_2$ and $\lVert \cdot \rVert_\infty$ to denote the $\ell_2$ and $\ell_\infty$ norm. Given a finite set $I,$ we denote $\text{Unif}(I)$ the uniform distribution over $I$. 

\subsection{Markov Games}

We consider the finite horizon general-sum Markov games $
(\mathcal{S}, H, \{\mathcal{A}_i\}^m_{i=1}, \{\mathbb{P}_{h}\}_{h=1}^H, \{r_{h,i}\}_{h,i=1}^{H,m})$. Here, $\mathcal{S}$ is the state space, $H$ denotes the time horizon, and $\mathcal{A}_i$ stands for the action space of the $i$-th player. We let $\mathcal{A} = \prod_{i=1}^m\mathcal{A}_i$ be the joint action space and $\bm{a} = (a_1, a_2, \cdots, a_m) \in \mathcal{A}$ represent the joint action. Given $s\in \mathcal{S}$ and $\bm a\in \mathcal{A}$, $\mathbb{P}_h(\cdot|s,\bm{a})$ denotes the transition probability and $r_{h,i} (s,\bm a)\in [0, 1]$ denotes the deterministic reward received by the $i$-th player at time-step $h$. We denote $S:=|\mathcal{S}|, A_i:=|\mathcal{A}_i|, A:= \max_i A_i$ the cardinality of state and action spaces.
Throughout the paper, we assume the considered Markov games always start at some fixed initial state $s_1$.\footnote{This assumption can be easily generalized to the setting where the initial state is sampled from some fixed distribution $\mu$, as in \cite{cui2023breaking,jin2021v}.}

\textbf{Markov Policy.} In this work, we consider the learning of Markov policies. A Markov policy selects action depending on historical information only through the current state $s$ and time step $h$. The Markov policy of player $i$ can be represented as $\pi_i := \{\pi_{h,i}\}_{h\in [H]}$ with $\pi_{h,i}: \mathcal{S} \rightarrow \Delta(\mathcal{A}_i)$. The joint policy of all agents is denoted by $\pi = (\pi_1,\dots,\pi_m)$. For a joint policy $\pi$, we denote $\pi_{-i}$ the joint policy excluding the one of player $i$. 
For $\pi_i^{\prime}: \mathcal{S}\times [H] \to \Delta(\mathcal{A}_i)$, we use $\pi_i^{\prime} \times \pi_{-i}$ to describe the policy where all players except player $i$ execute the joint policy $\pi_{-i}$ while player $i$ independently deploys policy $\pi_i^{\prime}$.

\textbf{Value function.} For a policy $\pi$, the value function $V_{h,i}^{\pi}: \mathcal{S} \rightarrow \mathbb{R}$ of the $i$-th player under a Markov policy $\pi$ at step $h$ is defined as
\begin{align}
    V_{h,i}^{\pi}(s)=\mathbb{E}\big[\sum_{t=h}^Hr_{t,i}(s_t,\bm{a}_t)|s_h = s\big], \quad \forall s \in \mathcal{S},
\end{align}
where the expectation is taken over the state transition and the randomness of policy $\pi$.  The $V^{\pi}_{h,i}$ satisfies the \textit{Bellman equation}: 
\begin{equation}\label{eq-V-bellman}\begin{aligned}V^{\pi}_{h,i}(s) &= \mathbb{E}_{\bm a \sim \pi}[Q^{\pi}_{h,i}(s,\bm a)],\\
Q_{h,i}^{\pi}(s,\bm a): &= r_{h,i}(s,\bm a) + \mathbb{P}_{h} V_{h+1,i}^{\pi}(s,\bm a),\end{aligned}\end{equation}
where $\mathbb{P}_h V_{h+1,i}^{\pi}(s,\bm a): = \E_{s' \sim \mathbb{P}_h(\cdot \lvert s,\bm a)} [V_{h+1,i}^{\pi}(s')]$.

Given other players acting according to $\pi_{-i}$, the \textit{best response policy} of the $i$-th player is the policy independent of the randomness of $\pi_{-i}$ achieving $V^{\dagger,\pi_{-i}}_{h,i}(s) := \max_{\pi_i^\prime}V_{h,i}^{\pi_i^{\prime} \times \pi_{-i}}(s).$ With the dynamic satisfied similar to (\ref{eq-V-bellman}), 
\begin{align*}
    V_{h,i}^{\dagger, \tilde{\pi}_{-i}}(s) = \max_{a}\big\{ r_{h,i}^{\tilde{\pi}_{h,-i}}(s,a) + \mathbb{P}_{h}^{\tilde{\pi}_{-i}}V_{h+1,i}^{\dagger, \tilde{\pi}_{-i}}(s,a)\big\},
\end{align*}
and $\mathbb{P}_h ^{\tilde{\pi}_{-i}}V(s,a): = \E_{\bm a_{-i} \sim \tilde{\pi}_{-i}} [\E_{s' \sim \mathbb{P}_h(\cdot \lvert s,a,\bm{a}_{-i})}[V(s')]]$.

\textbf{Nash equilibrium(NE).}
A product Markov policy $\pi = \pi_1 \times \cdots \times \pi_m$ is a Markov Nash equilibrium at state $s_1$ if $V_{1,i}^{\pi}(s_1) = V_{1,i}^{\dagger,\pi_{-i}}(s_1), \forall i \in [m].$

\textbf{Coarse correlated equilibrium(CCE).} 
A joint Markov policy $\pi$ is a Markov CCE at a state $s_1$ if $V_{1,i}^{\pi}(s_1) \geq V_{1,i}^{\dagger,\pi_{-i}}(s_1), \forall i \in [m].$ In this paper, we study the efficient learning of an $\varepsilon$-Markov CCE policy $\pi$ satisfying: 
\begin{align}
\max\limits_{i\in[m]}\{V_{1,i}^{\dagger,\pi_{-i}}(s_1)-V_{1,i}^{\pi}(s_1)\} \le \varepsilon.
\end{align}

Obviously, for general-sum Markov games, a Markov NE is also a Markov CCE. Further more, in two player zero-sum games, NE and CCE are equivalent.  For multi-player general-sum Markov games, computing the NE is statistically intractable. Therefore, we resort to the weaker and more relaxed equilibrium CCE, which can be calculated in polynomial computational time for general-sum Markov games \cite{Papadimitriou2008}. Still, it might be challenging for finding such an optimal relaxed equilibrium. We consider the approximated sub-optimal notation, $\varepsilon$-Markov CCE. In this work, our goal is to compute an $\varepsilon$-Markov CCE for the game with as few samples as possible.

\subsection{RL with Different Sampling Protocols}
 Given Markov Games, the learner does not have access to the underlying transition probabilities $\{\mathbb{P}_h\}$ and the reward functions $\{r_{h,i}\}$, but is assumed access to a random simulator. Whenever the learner queries the simulator with $(s,\bm{a},h)\in \mathcal{S}\times\mathcal{A}\times [H]$, he receives an independent sample $s'$ drawn from $\mathbb{P}_h(\cdot \lvert s,\bm{a})$. Based on the accessible range of state-action pairs using a simulator, we clarify three different sampling protocols typically used in RL as in \cite{yin2022efficient}:

\textbf{Online Access.} The learner can only interact with the simulator (environment) in real-time, and the state can be either reset to an initial state or transit to the next state given the current state and an action.

\textbf{Local Access.} The learner can query the simulator with any previously visited state paired with an arbitrary action.

\textbf{Random Access. } The learner can query the simulator with arbitrary state-action pairs.
Note that the random access model is often referred to as the generative model in the RL literature \citep{zhang2020generative, li2022minimax,li2020generative}.

The online access protocol imposes the least stringent requirement for accessing the simulator, whereas random access is the most restrictive assumption. 
The local access assumption, which is the central focus of this paper, is stronger than the online access protocol but more practical than the random access assumption. 
It has been successfully applied in the design of large-scale RL algorithms for practical problems, as demonstrated by previous studies \citep{Yin2023SampleED,Tavakoli2020ExploringRD,Ecoffet2019GoExploreAN,lan2023can}. 
In this paper, we show that the local access assumption can lead to improved sample complexity bounds compared to the online access setting.

\subsection{Independent Function Approximation}

Throughout this paper, we make the following assumption about the Markov Games:
\begin{assumption}[$\nu$-misspecified independent linear  MDP]\label{assumption-MDP}
Given a policy class $\Pi$ of interest, each player $i$ is able to access a feature map $\phi_i: \mathcal{S} \times \mathcal{A}_i \to \mathbb{R}^{d}$ with $\max_{s\in \mathcal{S},a\in \mathcal{A}_i}\lVert \phi_i(s,a) \rVert_2 \leq 1$. And there exists some $\nu >0$ so that for any $h\in [H]$ and $V: \mathcal{S} \to [0,H+1-h]$,  \begin{align}\label{eq-linear-approximation}
    \sup_{\pi \in \Pi} \min_{\lVert \theta\rVert_2 \leq H \sqrt{d} } \big\lVert  Q^{\pi_{-i},V}_{h,i}(\cdot,\cdot)  - \phi_i(\cdot,\cdot)^\top \theta\big\rVert_\infty \leq \nu.
\end{align}
where $Q^{\pi_{-i},V}_{h,i}(s,a):=\mathbb{E}_{\bm{a}_{-i}\sim \pi_{h,-i}(\cdot \lvert s)}\big[r_{h,i}(s,a,\bm{a}_{-i})+\mathbb{E}_{s^{\prime} \sim \mathbb{P}_h(\cdot \lvert s,a,\bm{a}_{-i})}[V(s^{\prime})]\big]$ is the marginal $Q$ function associated with $V$.
\end{assumption}

Assumption~\ref{assumption-MDP} asserts that for any $i\in [m]$ and $\pi \in \Pi$, if all the other players act according to $\pi_{-i}$, then the $i$-th player's environment is approximately linear MDP. This assumption extends the widely used linear MDP assumption in single-agent RL to multi-agent settings.

Compared to the centralized approximation approach used in prior works \citep{chen2021almost,xie2020learning,CisnerosVelarde2023FinitesampleGF}, which employs a $d \propto S\prod_{i=1}^m A_i$-dimensional linear function class to approximate a global $Q$-function for tabular Markov games, the independent approximation framework presented in Assumption~\ref{assumption-MDP} allows for the representation of the same environment with individual $Q$-functions of dimensions $d \propto S A$. This assumption avoids the need for the considered function class to have complexity proportional to the exponential of the number of agents.

As in \cite{cui2023breaking,wang2023breaking}, we restrict \eqref{eq-linear-approximation} to a particular policy $\Pi$. As discussed in Appendix~D of \cite{wang2023breaking}, if \eqref{eq-linear-approximation} holds with $\nu = 0$ for all $\Pi$, then the MG is essentially tabular. 
Since our algorithm design does not require prior knowledge of $\Pi$ , we defer the discussion of the policy class $\Pi$ considered in this paper in Appendix~\ref{appendix-policy-class}.

%% file: linear_game.tex
\allowdisplaybreaks

\section{Algorithm and Guarantees for Independent Linear Markov Games}

In this section, we present the Lin-Confident-FTRL algorithm for learning $\varepsilon$-CCE with local access to the simulator.
We then provide the sample complexity guarantee for this algorithm.
\subsection{The Lin-Confident-FTRL Algorithm}

We now describe the Lin-Confident-FTRL algorithm (Algorithm~\ref{alg:Main}). 

Our algorithm design is based on the idea that each agent maintains a core set of state-action pairs. 
The algorithm consists of two phases: the \textbf{policy learning phase} and the \textbf{rollout checking phase}. 
In the policy learning phase, each agent performs decentralized policy learning based on his own core set. 
In the rollout checking phase, the algorithm performs rollout with the learned policy to ensure the trajectory of the policy is well covered within the core set of each player. 
We will provide a detailed explanation of these two phases in Sections~\ref{policy learning} and~\ref{rollout checking}, respectively.

Before the policy learning phase, the $m$ players draw a joint trajectory of states $s_1, \dots, s_H$ of length $H$ by independently sampling actions following a uniform policy. Each player then initializes distinct core sets $\{\mathcal{D}_{h,i}\}_{h=1}^H$ with this trajectory through an exploration subroutine described below (Algorithm~\ref{alg:explore}).

\textbf{Core set expansion through an exploration subroutine. }
During the Explore subroutine at time $h$ with input state $s$, each agent $i$ iteratively appends state action pairs $(s,a)$ to its core set $\mathcal{D}_{h,i}$ and update $\Lambda_{h,i}$ until the coverage condition $\max_{a \in \mathcal{A}_i}\phi_i(s,a)^\top \Lambda_{h,i}^{-1} \phi_i(s,a) \leq \tau$ is met at state $s$.
Here $\tau$ is a predetermined threshold and $\Lambda_{h,i}^{-1}$ is the precision matrix corresponding to $\mathcal{D}_{h,i}$.

Given $\Lambda_{h,i},$ we define 
\begin{equation}
    \mathcal{C}_{h,i}:=\{s\in \mathcal{S}: \max_{a \in \mathcal{A}_i}\phi_i(s,a)^\top \Lambda_{h,i}^{-1} \phi_i(s,a) \leq \tau\}\end{equation} 
as the set of well-covered states for agent $i$ at step $h$. We refer 
$\mathcal{C}_h:= \cap_{i}\mathcal{C}_{h,i}$ as the confident state set for all agents. Note that for the implementation of the algorithm, it is not necessary to compute $\mathcal{C}_h$. We introduced it merely for the sake of describing the algorithm conventionally. In fact, the only operation that involves $\mathcal{C}_h$ is to determine whether a state $s$ belongs to it, and this can be done using solely the information from $\Lambda_{h,i}$.

\begin{figure}[!t]
 \removelatexerror
\begin{algorithm}[H]
\label{alg:explore}  
\caption{Explore$(s,h)$}
\textbf{Input: }{state $s$, time-step $h$}\\
\For{$i = 1$ \KwTo $m$}{
  \While{$\max_{a \in \mathcal{A}_i}\phi_i(s,a)^\top \Lambda_{h,i}^{-1} \phi_i(s,a) > \tau $}  
  { $\hat{a}_i = \text{argmax}_{a \in {A}_i} \phi_i(s,a)^\top \Lambda_{h,i}^{-1} \phi_i(s,a)  $ \\
    $\mathcal{D}_{h,i} \leftarrow \mathcal{D}_{h,i}  \cup \{(s,\hat{a}_i )\} $  \\
    ${\Lambda}_{h,i} \leftarrow {\Lambda}_{h,i}+ \phi_i(s,\hat{a}_i )\phi_i(s,\hat{a}_i )^\top$\\    
    } ${\mathcal{C}}_{h,i} \leftarrow \{s\in S: \max_{a\in \mathcal{A}_i} \phi(s,a)^\top \Lambda_{h,i}^{-1} \phi(s,a) \leq \tau \} $\textcolor{blue}{// the well-covered state set}
}
  ${\mathcal{C}}_h \leftarrow \cap_{i} \mathcal{C}_{h,i}$ 
\end{algorithm}
\end{figure}

\begin{figure}[!t]
 \removelatexerror

\begin{algorithm}[H]
  \label{alg:Main}
  \caption{Lin-Confident-FTRL}
  \SetKwInOut{Input}{Input}
  \SetKwInOut{Output}{output}
  \SetAlgoLined
  \textbf{Initialize Global variables:}  $ \mathcal{C}_{H+1} = \mathcal{S},\mathcal{C}_h = \emptyset, \forall h \in [H]$  and \begin{align*}&\widehat{V}_{h,i} = H +1 - h, \hat{V}^{\dagger}_{h,i} = H +1 - h, \mathcal{D}_{h,i} = \emptyset,\\
  &\Lambda_{h,i} = \lambda I,\quad \pi^1_{h,i}(\cdot \lvert s) = \text{Unif}(A_i),\quad \forall s,i,h.\end{align*}
  
  Sample a trajectory $\{s_1, \cdots, s_H\}$ of length $H$ with policy $\pi^1_{h,i}$

  \For{$h = 1$ \KwTo $H$}{\text{Explore}($s_h,h$) \textcolor{blue}{//See Algorithm~\ref{alg:explore}} }
\textcolor{blue}{//Policy Learning Phase}\\
  \For{$h = H$ \KwTo $1$}{
    Success $\leftarrow$ Multi-Agent-Learning($h$) .\textcolor{blue}{//See Algorithm~\ref{alg:UC-linear}}\\
      \If{Success = False}{
        Back to Line~5
        \textcolor{blue}{//Restart the loop from $h= H$.}
      }
      }
    { $\hat{\pi}_{h}\leftarrow \frac{1}{K}\sum_{k=1}^K  {\pi}_{h,1}^k\times \dots \times {\pi}_{h,m}^k, \forall h\in [H]$.}\\
    
  \textcolor{blue}{//Rollout Checking Phase}

    Success $\leftarrow$ Policy Rollout($\hat\pi,s_1,N$)\textcolor{blue}{//See Algorithm~\ref{alg:rollout}}
    
    \If{Success = False}
    {Return to Line~5} 
    
    \For{$i\in [m]$}
    {
      \For{$h = H$ \KwTo $1$}
      {
        Success $\leftarrow$ Single-Agent-Learning($h$, $i$, $\hat\pi_{h, -i}$).\textcolor{blue}{//See Algorithm~\ref{alg:linear-single}}\\
        \If{Success = False}{
          Back to Line~5
        }}
    }

    \For{$i\in [m]$}{
    Success $\leftarrow$ Policy Rollout($\hat\pi^\dagger_i\times \hat{\pi}_{-i},s_1,N$)  \\
    \If{Success = False}
    {Return to Line~5} 
    }
    \Return{ $\{\hat{\pi}_{h}\}_{h \in [H]}$ } 
\end{algorithm}
\end{figure}

During the subsequent policy learning and rollout checking phases, whenever a state outside the confident set is encountered, the exploration subroutine will be triggered to expand the core set and the learning process will be restarted.
Actually we have the following result regarding the cardinality of $\mathcal{D}_{h,i}$:

 \begin{lemma}[\cite{yin2022efficient}]\label{lem-core-set} For each $i$ and $h$, the size of the core set $\mathcal{D}_{h,i}$ will not exceed { \begin{align*}
    C_{\max}: = \frac{e}{e-1}\frac{1+\tau}{\tau}d\big(\log(1+\frac{1}{\tau})+\log(1+\frac{1}{\lambda}) \big)
\end{align*}}
\end{lemma}

As a corollary, both the number of calls to the Explore subroutine and the number of restarts are upper-bounded by $mHC_{\max}$.

\begin{figure}[!t]
 \removelatexerror
\begin{algorithm}[H]
  \label{alg:UC-linear}
  \caption{Multi-Agent-Learning}
  \SetAlgoLined
  \textbf{Input:}{ time-step $h$}\\
  \For{$k =1$ \KwTo $K$}{
    \For{$i = 1$ \KwTo $m$}{
      \For{$(\bar{s},\bar{a})\in \mathcal{D}_{h,i}$}{
      $(r, {s}') \leftarrow \text{local sampling}(h,i, \bar{s}, \bar{a}, \pi^k_{h,-i} )$ \textcolor{blue}{//See Algorithm~\ref{alg:LS}}\\
          \If{$ s' \notin \mathcal{C}_{h+1}$}{
          Explore$(s',h+1)$ \\
          \Return{False}    
          }     
     Compute $q_{h,i}^k(\bar{s},\bar{a}) = r + \hat{V}_{h+1,i}(s')$.  
    } 
    Update $Q^k_{h,i}(s,a)$ as in \eqref{eq-Q-update}.\\
    $\bar{{Q}}_{h,i}^k(s,a) \leftarrow \frac{k-1}{k}\bar{Q}_{h,i}^{k-1}(s,a)+\frac{1}{k}{Q}^k_{h,i}(s,a).$\\
     $\pi_{h,i}^{k+1}(a \lvert {s}) \leftarrow \frac{\exp(\eta_k \bar{Q}_{h,i}^k({s},a))}{\sum_{a'}\exp(\eta_k \bar{Q}_{h,i}^{k}({s},a'))}$.
     }}
     \textcolor{blue}{//Value estimation of $V_{h,i}^{\hat\pi}$ with $\hat{\pi}_h = \frac{1}{K}\sum_{k=1}^K \pi^k_{h,1}\times \dots \times \pi^k_{h,m}$}\\
 \For{$i = 1$ \KwTo $m$}{
Update $\hat{V}_{h,i}(s)$ as in \eqref{eq-V-est}
    }
\Return{True}
\end{algorithm}
\end{figure}

\begin{figure}[!t]
 \removelatexerror
\begin{algorithm}[H]
  \label{alg:LS}
  \caption{Local Sampling$(h, i, s, a, \pi_{-i})$}
  Draw an independent sample from the simulator:
  $$s' \sim \mathbb{P}_h(\cdot \lvert s, a, \bm{a}_{-i}),$$
  where $\bm{a}_{-i} \sim \pi_{h, -i}$\\
  \Return $(r_{h,i}(s, a,\bm{a}_{-i}), s')$ \textcolor{blue}{// the reward $\&$ transition pair given the sampled actions.}
\end{algorithm}
\end{figure}

\subsection{Policy Learning Phase}\label{policy learning}

After all agents have constructed the initial core sets based on the sampled trajectory,  they proceed to the policy learning phase by executing a multi-agent learning subroutine(Algorithm~\ref{alg:UC-linear}) recursively from $h = H$ to $h=1.$
To address the issue of multi-agency, we have incorporated the adaptive sampling strategy proposed in \cite{li2022minimax}, which operates under random access, into this subroutine. We have modified this approach by restricting the sampling to the core set of each agent $i$ instead of all $\mathcal{S}\times \mathcal{A}_i$ pairs. This is because our algorithm operates under local access and the core set provides enough information for efficient learning without revisiting unnecessary states and actions. 
 
\textbf{Multi-Agent Learning Subroutine. }
At the $k$-th iteration of Algorithm~\ref{alg:UC-linear}, each agent $i$ employs the local sampling subroutine (Algorithm~\ref{alg:LS}) over his core set, which returns a reward-state pair $(r, s')$. This design ensures that $q^k_{h,i}:= r+\hat{V}_{h+1,i}(s')$ provides a one-step estimation of $Q^{{\pi}_{-i}^k,\hat{V}_{h+1,i}}_{h,i}(s,a)$. 
If the estimators $q_{h,i}^k(\tilde{s},\tilde{a})$ are collected for all $(\bar{s},\bar{a})\in \mathcal{D}_{h,i}$ without restart, we proceed to update $Q_{h,i}^k$ via least square regression over the collected data:{ \small \begin{equation}\label{eq-Q-update}
    Q^k_{h,i}(s,a)=  \phi_i(s,a)^\top \Lambda_{h,i}^{-1} \sum_{(\tilde{s},\tilde{a})\in \mathcal{D}_{h,i}} \phi_i(\tilde{s},\tilde{a})q_{h,i}^k(\tilde{s},\tilde{a})
\end{equation}}
\vspace{-0.3cm}

and take policy iteration using the FTRL update \citep{Lattimore2020BanditA}, which has been widely adopted in the multi-agent game to break the curse of multi-agency \citep{li2022minimax,jin2021v,song2022learn}. After $K$ epochs, we obtain final policy $\{\pi^k_{h,i}\}_{k=1}^K$ and the estimated value {\small \begin{equation}\label{eq-V-est}
    \hat{V}_{h,i}(s) = \min\left\{ \frac{1}{K}\sum_{k=1}^{K}  \langle \pi_{h,i}^{k}, {Q}_{h,i}^{k}(s,\cdot) \rangle, H-h+1 \right\}
\end{equation} } under $\pi$ correspondingly.



\subsection{Rollout Checking Phase}\label{rollout checking}
If the policy $\hat\pi$ is learned without any restarts, then the Algorithm~\ref{alg:Main} will execute the final rollout checking procedure (Algorithm~\ref{alg:rollout}) to determine whether to output the learned policy $\hat\pi$ or not. Given any joint policy $\pi$, the rollout subroutine draws $N$ trajectories by employing $\pi$ for $N$ epochs. 
Whenever an uncertain state is met during the rollout routine, the algorithm will restart the policy learning phase with the updated confident set.
\begin{figure}[!t]
 \removelatexerror
\begin{algorithm}[H]
 \label{alg:rollout}
 \caption{Policy Rollout}

 \textbf{Input:}{ rollout policy $\pi$, initial state $s_1$, rollout times $N$ }  \\
 \For{$n \in [N]$ } 
 {
 Set $s'= s_1$\\
 \For{$h =1,\dots,H$ }
   {
       Sample $\bm a\sim {\pi}_{h}(s'),\quad s' \sim \mathbb{P}_h(\cdot \lvert s',\bm a)$.\\
       \If{$s'\notin \mathcal{C}_{h+1}$}
       {
         Explore$(s',h+1)$ \\  
      \Return {$\text{False}$}}
   }
}
 \Return {$\text{True}$}
\end{algorithm}
\end{figure}

\textbf{Necessity of rollout checking. } 
The rollout checking is necessary because the policy learning phase only considers information within $\mathcal{C}_h$, while the performance of a policy is determined by all the states encountered in its trajectory. Intuitively, the rollout subroutine ensures that the trajectory generated by the learned policy only contains states that are well covered by core sets, with a high probability.

Although the aforementioned rollout operation ensures that the trajectory of the joint policy $\hat{\pi}$ lies within well-covered states, this may not hold for best response policies. Specifically, for every player $i$ and their best response policy $\pi_i^\dagger$ given $\hat{\pi}_{-i}$, the trajectory of $\pi_i^\dagger \times \hat{\pi}_{-i}$ may lie outside of $\{\mathcal{C}_h\}$ with non-negligible probability. This motivates us to perform additional rollout for ${\pi}^{\dagger}_i \times \hat{\pi}_{-i}$. Since $\pi_i^\dagger$ is unknown without knowledge of the underlying transition kernels, we perform a single-agent learning subroutine to obtain an \textit{approximate best response} policy $\hat{\pi}_i^\dagger=\{\hat{\pi}_{h,i}^\dagger \}_{h \in [H]}$ and then take rollout for $\hat{\pi}_i^\dagger \times \hat{\pi}_{-i}$. As shown in the proof, dealing with these learned approximated best response policies is sufficient to provide the CCE guarantee for Algorithm~\ref{alg:Main}.

\begin{figure}[!t]
 \removelatexerror
\begin{algorithm}[H]
  \label{alg:linear-single}
  \SetKwInOut{Input}{Input}
  \SetKwInOut{Output}{output}
  \caption{Single-Agent-Learning}
  \SetAlgoLined
  \Input{time-step $h$,agent $i$, policy $\pi_{-i}$} 
      \For{$(\bar{s},\bar{a})\in \mathcal{D}_{h,i}$}{
      \For{$ k = 1$ \KwTo $K$ }{
      $(r, {s}') \leftarrow \text{local sampling}(h, i, \bar{s}, \bar{a}, \pi_{-i})$ \textcolor{blue}{//See Algorithm~\ref{alg:LS}}\\
          \If{$ s' \notin \mathcal{C}_{h+1}$}{
          Explore$(s',h+1)$ \\
          \Return{False}    
          }
    Compute $q_{h,i}^k(\bar{s},\bar{a}) = r + \hat{V}_{h+1,i}^\dagger(s')$.  
          }
    }
    $\hat{Q}_{h,i}^\dagger(s,a) \leftarrow \frac{1}{K}\phi_i(s,a)^\top \Lambda_{h,i}^{-1} \sum_{k=1}^K\sum_{(\tilde{s},\tilde{a})\in \mathcal{D}_{h,i}} \phi_i(\tilde{s},\tilde{a})q_{h,i}^k(\tilde{s},\tilde{a}) $\\
    $\hat{\pi}^{\dagger}_{h,i}(a\lvert \tilde s)\leftarrow \bm{1}\{ a = \text{argmax}Q_{h,i}^\dagger(s,\cdot)\}$ \\
$\hat{V}^{\dagger}_{h,i}(s)\leftarrow \max_{a} \hat{Q}^{\dagger}_{h,i}(s,a)$\\
\Return{True}
\end{algorithm}
\end{figure}

\vspace{-0.3cm}
\paragraph{Single-Agent Learning Subroutine.} 
To learn the best response for each agent $i$, we fix the other agents' policies and reduce the problem to a single-agent learning task. Specifically, we use Algorithm~\ref{alg:linear-single} to perform least squared value iteration backward in $h$. This subroutine can be seen as a finite-horizon version of the \textit{Confident-LSVI} algorithm proposed in \cite{hao2022confident} for single-agent learning under the local access model. Similar to other routines, the learning process restarts when encountering a new uncertain state. 


\subsection{Theoretical Results}

Now we state the theoretical result of Algorithm~\ref{alg:Main}, whose proof is deferred to Appendix~\ref{appendix-proof}.
\begin{theorem}\label{thm-local-access} Under Assumption~\ref{assumption-MDP}, Algorithm~\ref{alg:Main} with $N,K,\tau,\lambda = \text{Poly}(\log(S),d,H,\varepsilon^{-1},A)$, $\eta_k = \tilde{O}(k\min\{\sqrt{\log(S)/d} + \nu,1\}) ]^{-1} \text{Poly}(K,d,H))$  returns an $(\varepsilon +3\nu\sqrt{d}H)$-Markov CCE policy with probability at least $1-\delta$ with \\
i)
\vspace{-0.5cm}
\begin{align*}
   \tilde{O}\bigg(\frac{m^2d^3H^6}{\varepsilon^2} \min\big\{ d^{-1}{\log S}, A \big\}  \bigg) \end{align*}
query of samples under the local access model when $\min\big\{ d^{-1}{\log S}, A \big\} \leq \varepsilon^{-2}$,\\
ii)
\vspace{-0.5cm}
\begin{align*}
 \tilde{O}\bigg( m^2d^5H^{14}\varepsilon^{-6} \bigg)
\end{align*}
query of samples under the local access model when $\min\big\{ d^{-1}{\log S}, A \big\} > \varepsilon^{-2}$.

The detail of all parameter settings are leaved in Appendix~\ref{appendix-proof}.
\end{theorem}

\vspace{-0.2cm}


When compared to previous works operating under online access, \cite{cui2023breaking} attains $(\varepsilon+ \nu H)$-CCE with $\tilde{O}(d^4H^{10}m^4\varepsilon^{-4})$ samples. Meanwhile, \cite{wang2023breaking} achieves $\varepsilon$-CCE with $\tilde{O}(d^4H^6m^2\max_i A_i^5\varepsilon^{-2})$ samples under the realizability assumption ($\nu = 0$). In the scenario where $\min\{d^{-1}\log S,A \} \leq \varepsilon^{-4}H^6$, our result improves the dependency on parameters $d, H, m, \max_i A_i$, and $\varepsilon$. Conversely, in scenarios with extremely large $S, A$ values, our bound $\tilde{O}(m^2d^5H^{14}\varepsilon^{-6})$ fall short of those in \cite{cui2023breaking}. 

\textbf{Tighter Complexity Bound with Random Access. } Since the policy output by Lin-Confident-FTRL is only updated based on the information of the shared confident state set, in the analysis of the algorithm, we need to construct virtual algorithms that connect to Lin-Confident-FTRL on the confident state and have strong guarantees outside the set. Note that the virtual algorithms are solely intended for analytical purposes and will not be implemented under the local access model. Unlike prior works on single-agent RL \citep{yin2022efficient,hao2022confident}, where an ideal virtual algorithm is constructed using population values of $Q$ functions, we develop our virtual algorithms in an implementable manner under the random access model. As a bonus of our virtual algorithm analysis, we derive an algorithm that can operate directly under the random access model with a tighter sample complexity. We would state the result formally as follows:
\begin{theorem}\label{thm-generative}
   Under Assumption~\ref{assumption-MDP}, there exists a decentralized algorithm under random access model that returns a joint policy achieving $(\varepsilon+ 3\nu\sqrt{d}H)$-CCE  with probability at least $1-\delta$ and  $\tilde{O}(\min\{\varepsilon^{-2}dH^2,\frac{\log(S)}{d}, A\}d^2H^5m\varepsilon^{-2})$ sample complexity bound. The details of the algorithm design are leaved in Appendix~\ref{appendix-generative}.
\end{theorem}

\vspace{-0.2cm}

Under the more restrictive random access protocol, Theorem~\ref{thm-generative} suggests that there is an algorithm with a more precise dependency on all parameters compared to previous results. This proposed algorithm can be viewed as an analogue to Algorithm~\ref{alg:Main}. However, due to the relaxed sampling protocol, there's no need for restarts, which results in a savings of a factor $mdH$ when $\{\log(S)/d,A\}<\varepsilon^{-2}$ and a savings of $\varepsilon^{-2}mdH$ when $\{\log(S)/d,A\}\geq \varepsilon^{-2}$. We conjecture that even under the local access protocol, a more refined algorithmic design can avoid this additional restart cost. One possible approach might be to incorporate the \textit{Confident Approximate Policy Iteration} method from \cite{weisz2022confident} into our setting. This remains an avenue for future exploration. 
Lastly, note that in the tabular case, with modifications to the FTRL step-size and value estimation formula, the algorithm proposed in Appendix~\ref{appendix-generative} coincides with the algorithm proposed in~\cite{li2022minimax}, which achieves the minimax optimal sample complexity. 

\textbf{Sample Complexity without the knowledge of $\nu$.} While our selection of the FTRL stepsize $\eta_k$ in Theorem~\ref{thm-local-access} requires prior knowledge of the misspecification error $\nu$. When $\nu$ is unknown, we can select $\eta_k = \tilde{O}(kK^{-1/2}H\sqrt{d})$ in Algorithm~\ref{alg:Main}. Then the algorithm is still guaranteed  to output a $(\varepsilon+3\nu\sqrt{d}H)$-Markov CCE with $\tilde{O}(m^2d^3H^{6}\varepsilon^{-2})$ samples.


\textbf{Decentralized Implementation and Communication Cost.} \quad 
While we describe Algorithm~\ref{alg:Main} and its subroutines in a centralized manner, we remark that  
it can be implemented in a decentralized manner with limited communication. More precisely, during the running of the algorithm, each agent only need to observe its own rewards and actions. And communication between agents only occurs during the initialization procedure and every time a restart occurs.
We will discuss the decentralized implementation in Appendix~\ref{appendix-computation} and show that the total communication complexity is bounded by $\tilde{O}(mdH)$, which is identical to that of the \textit{PReFI} algorithm proposed in \cite{cui2023breaking}  and the \textit{AVPLR} algorithm proposed in \cite{wang2023breaking}.\footnote{We remark here both \cite{cui2023breaking} and \cite{wang2023breaking} also propose other fully decentralized algorithms, but with worse sample complexity bound.}


\section{Conclusion}

In this work, we have considered multi-agent Markov games with independent linear function approximation within both the random and local access models. Our proposed algorithm, \textit{Linear-Confident-FTRL}, effectively mitigates the challenges associated with multi-agency and circumvents the dependency on the action space for regimes where $S \lesssim e^{d\max_i A_i}$. Additionally, our theoretical analysis has lead to the development of a novel algorithm that offers enhanced sample complexity bounds for independent linear Markov games in the random access model.
Several compelling questions remain open for exploration: 

The first is to investigate the independent function approximation setting under weaker realizability assumptions. Second, designing an algorithm to attain $O(\varepsilon^{-2})$ sample complexity without polynomial dependency on the action space $A$ and logarithmic dependency on the state space $S$ remains an unresolved challenge and an interesting direction for future research.

\section*{Acknowledgements}
Jian-Feng Cai is partially supported by Hong Kong Research Grant Council(RGC) GRFs 16310620, 16306821, and 16307023, Hong Kong Innovation and Technology Fund MHP/009/20 and the Project of Hetao Shenzhen-HKUST Innovation Cooperation Zone HZQB-KCZYB-2020083. Jiheng Zhang is  supported by RGC GRF 16214121. Jian-Feng Cai and Yang Xiang are supported by the Project of Hetao Shenzhen-HKUST Innovation Cooperation Zone HZQB-KCZYB-2020083. Yang Wang is supported by RGC CRF 8730063 and Hong Kong Center of AI, Robotics and Electronics (HK CARE) for Prefabricated Construction.

%% file: appendix_policy_class.tex
\section{Discussion on Policy Class}\label{appendix-policy-class}
As pointed in Appendix~D of \cite{wang2023breaking}, if the Assumption~\ref{assumption-MDP} holds for all possible policy $\pi$, then the underlying game must be essentially tabular game. Thus it would be necessary to discuss the range of policies where our assumption holds. Actually, to ensure Theorem~\ref{thm-generative} holds, we need only the Assumption~\ref{assumption-MDP} holds for the policy class defined as following: 
\begin{align}\label{eq-soft-max-class}
    \Pi &= \bigg\{\prod_{i=1}^m \pi_i: \pi_i(s,a) \propto \exp(- \eta \phi_i(s,a)^\top \theta_i ): \eta \geq \eta_0,\theta_i \in \mathbb{R}^d \bigg\}
    \end{align} 
The above soft-max policy class is similar to those considered in \cite{cui2023breaking}, and it also contains the argmax policy considered in \cite{wang2023breaking} as $\eta \to +\infty$.
However, we would remark that while the considered policy class are similar, the independent linear MDP assumption made in our Assumption~\ref{assumption-MDP} is strictly stronger than the assumptions made in \cite{cui2023breaking} and \cite{wang2023breaking}.

On the other hand, to ensure the result in Theorem~\ref{thm-local-access} holds, we need Assumption~\ref{assumption-MDP} holds for the following class, which is a bit complex than \eqref{eq-soft-max-class}:
\begin{align}\label{eq-soft-max-class-local}
    \Pi &= \left\{\prod_{i=1}^m \pi_i: \pi_i(s,a) \propto \left\{\begin{matrix}\exp(- \eta \phi_i(s,a)^\top \theta_i ),\quad s \in \mathcal{C}_i\\
    \exp(- \eta \phi_i(s,a)^\top \theta_i' ),\quad s \notin \mathcal{C}_i
    \end{matrix}\right.: \eta \geq \eta_0,\theta_i,\theta_i'\in \mathbb{R}^d, \mathcal{C}_i \subset \mathcal{S} \right\}
    \end{align} 
The class defined in \eqref{eq-soft-max-class-local} can be seen as an extension of \eqref{eq-soft-max-class} in the sense that when we divide the state space into two non-overlapping subsets, the policy is a soft-max policy over each subset. Although the policies generated by our main algorithm always lie in \eqref{eq-soft-max-class}, our results in the local access setting require Assumption~\ref{assumption-MDP} to hold over \eqref{eq-soft-max-class-local} due to technical reasons in our analysis. We believe that this assumption can be weakened, which we leave as a future direction.

%% file: appendix_computation.tex
\section{Discussion on the Communication Cost}\label{appendix-computation}
To discuss the communication cost, we would present the decentralized implementation of Algorithm~\ref{alg:Main}. During the implementation, only the knowledge of $\{\lvert \mathcal{D}_{h,i} \rvert\}_{h\in [H],i\in [m]}$ are need to known to each learner.
\begin{itemize}
    \item \textbf{At the beginning of the algorithm,} each agent independently keep the coreset $\mathcal{D}_{h,i} = \emptyset$ and share the same random seed. 
    \item \textbf{When the algorithm is restarted,} each agent will share the coreset size $\{\mathcal{D}_{h,i}\}_{h\in [H], i \in [m]}$.  Then for each agent $j$, until some agent $j$ meets a new state $s' \notin \mathcal{C}_{h,j}$ at some $h$, each agent can play action only with the knowledge of $\{\lvert \mathcal{D}_{h,i}\rvert \}_{i\in [m], h \in [H]}$ as the following:
     \begin{itemize} 
        \item To independently implement the line~7 to line~12 of the Algorithm~\ref{alg:Main}, it is sufficient to implement Algorithm~3 independently. In the inner loop of Algorithm~3 with the loop-index $k,i,\bar{s},\bar{a}$, the $j$-th agent play $\pi^k_{h,j}(\cdot \lvert \bar{s})$ when $i\neq j$ and play $\bar{a}$ when $i = j$. When $i = j$, the $j$-th agent will also update his policy and $Q,V$ functions as line~12 to line~19 in Algorithm~3.
        \item Line~13 can be implemented independently since they have communicated the shared random seed.
        \item To implement Line~15 of the algorithm, each agent $j$ just needs to play $\frac{1}{K}\sum_{k=1}^K\pi_{h,i}^k(\cdot \lvert s')$ with the shared random seed for $N$ epoches.
        \item To implement the loop in  Line~19 to Line~26 with loop index $i$ and inner loop index $\bar{s},\bar{a},k$ in Algorithm~6, the agent $j$ play action $\frac{1}{K}\sum_{k'=1}^K \pi(\cdot \lvert \bar{s})$ if $i\neq j$ and play $\bar{a}$ if $i = j$. When $i = j$, the $j$-th agent will also update his policy and $Q,V$ functions as line~11 to line~13 in Algorithm~5.
        \item The policy rollout loop in Line~28 of Algorithm~2 can be implemented in a similar way as in Line~15.
     \end{itemize}
    \item During the algorithm, \textbf{if some agent $j$ firstly meet some $s'\notin \mathcal{C}_{h+1,j}$, he will send the restarting signal to each agents}, after receiving such signal, each agent take the explore procedure in Algorithm~1 independently, take restart the learning procedure.    
\end{itemize}

In the above procedure, the communication only occurs in the initialization and restarting, thus is at most $\tilde{O}(mdH)$ times.

%% file: appendix_local.tex
\section{Proof of Theorem~\ref{thm-local-access}}\label{appendix-proof}

\subsection{The Virtual Algorithm}\label{sec-virtual-alg}
 As demonstrated by \cite{hao2022confident} in the single-agent case, the core-set-based update only guarantees performance within the core sets. However, it is essential to consider the information outside the sets for comprehensive analysis. To address this, we introduce a virtual algorithm, Lin-Confident-FTRL-Virtual (Algorithm~\ref{alg:Virtual}), which employs the same update as Algorithm~\ref{alg:Main} within the core sets and also provides good performance guarantees outside them. It is important to note that this virtual algorithm is solely for analytical purposes and will not be implemented in practice. We explain how the virtual algorithm is coupled with the main algorithm and provide additional comments below.

  \textbf{No restart but update of the coreset}\quad  {We run Algorithm~\ref{alg:Virtual} for $mHC_{\text{max}}$ epochs. During each epoch, the virtual algorithm employs similar subroutines to the Algorithm~\ref{alg:Main} except that 
it does not halt and restart when encountering a new state $s'$ not in $\mathcal{C}_h$ during the iteration. Instead, it continues the $K$-step iteration and returns a policy. Upon initially encountering such a state $s'$, the algorithm explores the state and add a subset of $\{s'\}\times A$ to the core sets for use in the next epoch.

\begin{figure}[!t]
 \removelatexerror
\resizebox{0.95\textwidth}{!}{
\begin{algorithm}[H]
  \label{alg:Virtual}
  \caption{Lin-Confident-FTRL-Virtual}
  \SetKwInOut{Input}{Input}
  \SetKwInOut{Output}{output}
  \SetAlgoLined
  \textbf{Initialize Global variables:}  FirstMeet $=$ True, $ \mathcal{C}_{H+1} = \mathcal{S},\mathcal{C}_h = \emptyset, \forall h \in [H]$  and \begin{align*}&\tilde{V}_{h,i} = H +1 - h,\tilde{V}^{\dagger}_{h,i} = H +1 - h,  \mathcal{D}_{h,i} = \emptyset,\\
  &\Lambda_{h,i} = \lambda I,\quad \forall s,i,h.\end{align*}
  
  Sample the same trajectory $\{s_1, \cdots, s_H\}$ of length $H$ and obtain the same initialized core sets as Line $2$--$4$ from Algorithm~\ref{alg:Main}

\textcolor{blue}{//Policy Learning Phase}\\
\For {$l=1$ \KwTo $ mH\mathcal{C}_{max}$}{
FirstMeet $=$ True \textcolor{blue}{//The $l$-th epoch correspondes the $l$-th restart of Algorithm~\ref{alg:Main} }

  \For{$h = H$ \KwTo $1$}{
    Multi-Agent-Learning-Virtual($h$) .\textcolor{blue}{//See Algorithm~\ref{alg:UC-linear-Virtual}}\\
      }
   { $\tilde{\pi}_{h}\leftarrow \frac{1}{K}\sum_{k=1}^K  \tilde{\pi}_{h,1}^k\times \dots \times \tilde{\pi}_{h,m}^k, \forall h\in [H]$.}
    
  \textcolor{blue}{//Rollout Checking Phase}

     Policy-Rollout-Virtual($\tilde\pi,s_1,N$)\textcolor{blue}{//See Algorithm~\ref{alg:rollout-Virtual}}

    \For{$i\in [m]$}
    {
      \For{$h = H$ \KwTo $1$}
      {
        Single-Agent-Learning-Virtual($h$, $i$, $\tilde\pi_{h, -i}$).\textcolor{blue}{//See Algorithm~\ref{alg:linear-single-Virtual}}\\
      }
    }

    \For{$i\in [m]$}{
    Policy-Rollout-Virtual($\tilde\pi_i^\dagger \times \tilde{\pi}_{-i},s_1,N$)  \\
    
    }
    \Return{ $\{\tilde{\pi}_{h}\}_{h \in [H]}$ }}
\end{algorithm}}
\end{figure}

\vspace{-0.2cm}

   

    

   \textbf{Coupled Simulator}\quad The virtual algorithm is coupled with the main algorithm in the following way: before the discovery of a new state  in each epoch, the simulator in the virtual algorithm generates the same action from random policies and the same trajectory of transition as those of Algorithm~\ref{alg:Main}. This coupled dynamic, combined with the condition that core sets are updated only upon the initial encounter with a new state, ensures that at the start of the $l$-th restart, the core sets of Algorithm~\ref{alg:Main}
} are identical to those of the $l$-th epoch of the virtual algorithm. Additionally, since the virtual $Q$ function is updated in the same manner as the main algorithm for 
states in core sets, the virtual policy in the 
$l$-th epoch is equivalent to the main policy in core sets at the $l$-th restart before encountering the first uncertain state in that epoch. In particular, there exists some $1\leq \tau \leq C_{\max}$ such that the main policy is identical to the virtual policy for every $h$ and $s\in \mathcal{C}_{h}$.

\paragraph{Virtual policy iteration outside the core sets} Besides the core sets, the virtual algorithm maintains in addition a collection of complementary sets $\tilde{\mathcal{D}}_{h,i}\setminus \mathcal{D}_{h,i}$ satisfying the confident state set of $(\tilde{\mathcal{D}}_{h,i}\setminus \mathcal{D}_{h,i}) \cup {\mathcal{D}}_{h,i} $ is $S$.  Obviously $ \tilde{\mathcal{D}}_{h,i}$ is also measurable with respect to the information collected up to finishing $l$-th epoch. The virtual algorithm can 
 query the unvisited states on $\tilde{\mathcal{D}}_{h,i}\setminus \mathcal{D}_{h,i}$, which implies that the virtual algorithm has random access to the simulator. The virtual algorithm samples over $\tilde{\mathcal{D}}_{h,i}\setminus \mathcal{D}_{h,i}$ and perform least square regression to update $Q$ functions and virtual policies for states outside $\mathcal{C}_h$. 
 Note again that although the virtual algorithm is assumed random access to the simulator, it serves only as a means for analysis and is never implemented. 
\begin{algorithm}[H]
 \label{alg:rollout-Virtual}
 \caption{Policy-Rollout-Virtual}

 \textbf{Input:}{rollout policy $\pi$, initial state $s_1$, rollout times $N$ }  \\
 \For{$n \in [N]$}
 {
 Set $s'= s_1$\\
 \For{$h =1,\dots,H$ }
   { \If{FirstMeet = True}{
      Obtain the same $(\bm a, {s}')$  as the sampled pair from Line $5$ of Algorithm~\ref{alg:rollout} within the same restarting epoch of Algorithm~\ref{alg:Main}  \\
          \If{$ s' \notin \mathcal{C}_{h+1}$}{
          Explore$(s',h+1)$ \\
          \textit{FirstMeet} = \textit{False} 
          }}
          \Else{  Sample $\bm a\sim {\pi}_{h}(s'),\quad s' \sim \mathbb{P}_h(\cdot \lvert s',\bm a)$.}
     
}
}

\end{algorithm}

 \begin{algorithm}
  \label{alg:UC-linear-Virtual}
  \caption{Multi-Agent-Learning-Virtual}
  \SetAlgoLined
  \textbf{Input:}{time-step $h$}\\
  \textbf{Initialize:} $\tilde{\mathcal{D}}_{h,i},  \tilde{\Lambda}_{h,i} = \Lambda_{h,i}, i\in [m]$\\
  
  \For{$i=1$ \KwTo $m$ }{  \While{$\max_{(\tilde{s},a) \in \mathcal{S}\times\mathcal{A}_i}\phi_i(\tilde{s},a)^\top \tilde{\Lambda}_{h,i}^{-1} \phi_i(\tilde{s},a) > \tau $}  
  { $(\hat{s},\hat{a}_i) = \text{argmax}_{(\tilde{s},a) \in \mathcal{S}\times {A}_i} \phi_i(\tilde{s},a)^\top \tilde{\Lambda}_{h,i}^{-1} \phi_i(\tilde{s},a) $ \\
    $\tilde{\mathcal{D}}_{h,i} \leftarrow \tilde{\mathcal{D}}_{h,i}  \cup \{(\hat{s},\hat{a}_i )\} $  \\
    $\tilde{\Lambda}_{h,i} \leftarrow \tilde{\Lambda}_{h,i}+ \phi_i(\tilde{s},\hat{a}_i )\phi_i(\tilde{s},\hat{a}_i )^\top$\\    
    }}
  \For{$k =1$ \KwTo $K$}{
    \For{$i = 1$ \KwTo $m$}{
      \For{$(\bar{s},\bar{a})\in \mathcal{D}_{h,i}$}{ \If{FirstMeet = True}{
      Obtain the same $(r, {s}')$ as the sampled pair from Line $5$ of Algorithm~\ref{alg:UC-linear} within the same restarting epoch of Algorithm~\ref{alg:Main}  \\
          \If{$ s' \notin \mathcal{C}_{h+1}$}{
          Explore$(s', h+1)$ \\
          \textit{FirstMeet} = \textit{False} 
          }}
          \Else{ $(r, {s}') \leftarrow \text{local sampling}(h,i, \bar{s}, \bar{a}, \tilde{\pi}^k_{h,-i} )$}
     Compute $q_{h,i}^k(\bar{s},\bar{a}) = r + \tilde{V}_{h+1,i}(s')$.  
    } 
    $\tilde{Q}^k_{h,i}(s,a)\leftarrow \phi_i(s,a)^\top \Lambda_{h,i}^{-1} \sum_{(\tilde{s},\tilde{a})\in \mathcal{D}_{h,i}} \phi_i(\tilde{s},\tilde{a})q_{h,i}^k(\tilde{s},\tilde{a}) $.\\
    $\bar{{Q}}_{h,i}^k(s,a) \leftarrow \frac{k-1}{k}\bar{Q}_{h,i}^{k-1}(s,a)+\frac{1}{k}{\tilde{Q}}^k_{h,i}(s,a).$\\
     $\tilde{\pi}_{h,i}^{k+1}(a \lvert {s}) \leftarrow \frac{\exp(\eta_k \bar{Q}_{h,i}^k({s},a))}{\sum_{a'}\exp(\eta_k \bar{Q}_{h,i}^{k}({s},a'))}$.
     }
    \For{$i = 1$ \KwTo $m$}{
         \For{$(\bar{s},\bar{a})\in \tilde{\mathcal{D}}_{h,i}\setminus \mathcal{D}_{h,i}$}{
         $(r, {s}') \leftarrow \text{local sampling}(i, \bar{s}, \bar{a}, \tilde{\pi}^k_h )$ \textcolor{blue}{//See Algorithm~\ref{alg:LS}}\\
        Compute $q_{h,i}^k(\bar{s},\bar{a}) = r + \tilde{V}_{h+1,i}(s')$.  
       }
       $\tilde{Q}^k_{h,i}(s,a)\leftarrow \phi_i(s,a)^\top \Lambda_{h,i}^{-1} \sum_{(\tilde{s},\tilde{a})\in \mathcal{\tilde{D}}_{h,i}} \phi_i(\tilde{s},\tilde{a})q_{h,i}^k(\tilde{s},\tilde{a})$ \textcolor{blue}{for $s\in S\setminus \mathcal{C}_h$}.\\
       $\bar{{Q}}_{h,i}^k(s,a) \leftarrow \frac{k-1}{k}\bar{Q}_{h,i}^{k-1}(s,a)+\frac{1}{k}{\tilde{Q}}^k_{h,i}(s,a).$ \textcolor{blue}{for $s\in S\setminus \mathcal{C}_h$.}\\
       $\tilde{\pi}_{h,i}^{k+1}(a \lvert {s}) \leftarrow \frac{\exp(\eta_k \bar{Q}_{h,i}^k({s},a))}{\sum_{a'}\exp(\eta_k \bar{Q}_{h,i}^{k}({s},a'))}$  \textcolor{blue}{for $s\in S\setminus \mathcal{C}_h$}.
    }}
     \textcolor{blue}{//Value estimation of $\tilde{V}_{h,i}^{\tilde\pi}$ with $\tilde{\pi}_h = \frac{1}{K}\sum_{k=1}^K \tilde{\pi}^k_{h,1}\times \dots \times \tilde{\pi}^k_{h,m}$}\\
 \For{$i = 1$ \KwTo $m$}{
$\tilde{V}_{h,i}(s) \leftarrow \min\left\{ \frac{1}{K}\sum_{k=1}^{K}  \langle \tilde{\pi}_{h,i}^{k}, \tilde{Q}_{h,i}^{k}(s,\cdot) \rangle, H-h+1 \right\}$  
    }

\end{algorithm}

\begin{algorithm}
\label{alg:linear-single-Virtual}
  \SetKwInOut{Input}{Input}
  \SetKwInOut{Output}{output}
  \caption{Single-Agent-Learning-Virtual}
  \SetAlgoLined
  \Input{time-step $h$,agent $i$, policy $\pi_{-i}$} \textcolor{blue}{// Also inherit the coresets $\mathcal{D}_{h,i},\tilde{\mathcal{D}}_{h,i}$ from Algorithm~8}\\
      \For{$(\bar{s},\bar{a})\in \mathcal{D}_{h,i}$}{
      \For{$ k = 1$ \KwTo $K$ }{
       \If{FirstMeet = True}{
      Obtain the same $(r, {s}')$  as the sampled pair from Line $3$ of Algorithm~\ref{alg:linear-single} within the same restarting epoch of Algorithm~\ref{alg:Main}  \\
          \If{$ s' \notin \mathcal{C}_{h+1}$}{
          Explore$(s',h+1)$ \\
          \textit{FirstMeet} = \textit{False} 
          }}
          \Else{   $(r, {s}') \leftarrow \text{local sampling}(h, i, \bar{s}, \bar{a}, \pi_{-i})$ \textcolor{blue}{//See Algorithm~\ref{alg:LS}}\\}

    Compute $q_{h,i}^k(\bar{s},\bar{a}) = r + \tilde{V}_{h+1,i}^\dagger(s')$.  
          }
    }
    $\tilde{Q}_{h,i}^\dagger(s,a) \leftarrow \frac{1}{K}\phi_i(s,a)^\top \Lambda_{h,i}^{-1} \sum_{k=1}^K\sum_{(\tilde{s},\tilde{a})\in \mathcal{D}_{h,i}} \phi_i(\tilde{s},\tilde{a})q_{h,i}^k(\tilde{s},\tilde{a}) $\\
    $\tilde{\pi}^{\dagger}_{h,i}(a\lvert \tilde s)\leftarrow \bm{1}\{ a = 
    \text{argmax}\tilde{Q}_{h,i}^\dagger(s,\cdot)\}$ \\
$\tilde{V}^{\dagger}_{h,i}(s)\leftarrow \max_{a} \tilde{Q}^{\dagger}_{h,i}(s,a)$\\
\textbf{repeat} the loop in line~2 with $\tilde{\mathcal{D}}_{h,i}\setminus \mathcal{D}_{h,i}$ and update $\tilde{Q}_{h,i}^\dagger,\tilde{\pi}_{h,i}^\dagger,\tilde{V}_{h,i}^\dagger$ using the collected data as in line~17 to line~19 \textcolor{blue}{ over $s\in S\setminus \mathcal{C}_h$}.
\end{algorithm}

\subsection{Analysis of the Virtual Algorithm}

For any $1\leq \ell\leq mHC_{\max}$, denote $\mathcal{F}^{\ell - 1}$ the $\sigma$-algebra generated by all actions and transitions before the $\ell$-th epoch. If it holds that conditioned on $\mathcal{F}^{\ell-1}$, the policy $\{\tilde{\pi}_{h}^{\ell}\}_{h\in [H]}$ outputted by the $\ell$-th epoch satisfies \begin{align}\label{eq-epoch-prob}
  \mathbb{P}\big( V_{1,i}^{\dagger, \tilde{\pi}_{-i}^\ell} - V_{1,i}^{\tilde{\pi}^\ell} \gtrsim H\nu \sqrt{d}+ H^2 \sqrt{\tau}\big(\sqrt{\frac{ \log(S) }{K }} \wedge \sqrt{ d(\frac{A}{K} \wedge 1 )}\big) +\gamma H^2 \sqrt{\frac{2\log A_i}{K}}\big) \leq \frac{\delta}{mHC_{\max}}.
\end{align}
Then it holds that \begin{align*}
  &\mathbb{P}\big( V_{1,i}^{\dagger, \tilde{\pi}_{-i}^\ell} - V_{1,i}^{\tilde{\pi}^\ell} \gtrsim H\nu \sqrt{d}+ H^2 \sqrt{\tau}\big(\sqrt{\frac{ \log(S) }{K }} \wedge \sqrt{ d(\frac{A}{K} \wedge 1 )} + \big) +\gamma H^2 \sqrt{\frac{2\log A_i}{K}}, \exists 1\leq \ell \leq mHC_{\max} \big)\\
  \leq& \E[ \sum_{\ell =1}^{mHC_{\max}} \bm{1}\{V_{1,i}^{\dagger, \tilde{\pi}_{-i}^\ell} - V_{1,i}^{\tilde{\pi}^\ell} \gtrsim H\nu \sqrt{d}+ H^2 \sqrt{\tau}\big(\sqrt{\frac{ \log(S) }{K }} \wedge \sqrt{ d(\frac{A}{K} \wedge 1 )}\big) + \gamma H^2 \sqrt{\frac{2\log A_i}{K}}\}  ]\\
  =& \E[ \sum_{\ell =1}^{mHC_{\max}}\E[\bm{1}\{V_{1,i}^{\dagger, \tilde{\pi}_{-i}^\ell} - V_{1,i}^{\tilde{\pi}^\ell} \gtrsim H\nu \sqrt{d}+ H^2 \sqrt{\tau}\big(\sqrt{\frac{ \log(S) }{K }} \wedge \sqrt{ d(\frac{A}{K} \wedge 1 )}\big) + \gamma H^2 \sqrt{\frac{2\log A_i}{K}}\} \lvert \mathcal{F}^{\ell -1}] ]\leq \delta.
\end{align*}

Then Let $H^2 \sqrt{\tau}\big(\sqrt{\frac{\log(S)}{K }} \wedge \sqrt{ d(\frac{A}{K} \wedge 1 )} + \gamma H^2 \sqrt{\frac{2\log A_i}{K}} \lesssim \varepsilon$, we have with probability at least $1 - \delta$,
\begin{align*}
    V_{1,i}^{\dagger, \tilde{\pi}_{-i}^\ell} - V_{1,i}^{\tilde{\pi}^\ell} \le \varepsilon +3 H\nu\sqrt{d}, \quad 1 \leq l \leq mHC_{\max}.
\end{align*}

where the coefficient $3$ for $H\nu \sqrt{d}$ is from combining \eqref{eq-I1-bound} and \eqref{eq-I2-bound}. And we select corresponding parameters $K, \tau$ here for $H^2 \sqrt{\tau}\big(\sqrt{\frac{ \log(S) }{K }} \wedge \sqrt{ d(\frac{A}{K} \wedge 1 )} + \gamma H^2 \sqrt{\frac{2\log A_i}{K}} \lesssim \varepsilon$.

When $\min\{{d^{-1}\log{S}, A}\} \le \varepsilon^{-2}$, select $K = \tilde{O}\big(H^4d\varepsilon^{-2} \min\{{d^{-1}\log{S}, A}\}\big)$, $\tau = 1$, then it holds
\begin{align*}
   H^2 \sqrt{\tau}\big(\sqrt{\frac{ \log(S) }{K }} \wedge \sqrt{ d(\frac{A}{K} \wedge 1 )} + \gamma H^2 \sqrt{\frac{2\log A_i}{K}} \lesssim \varepsilon + c H\nu \sqrt{d}.
\end{align*}
The corresponding query of samples is 
$m^2H^2KC_{\max}^2 = \tilde{O}\big(m^2H^6d^3\varepsilon^{-2}\min\{{d^{-1}\log{S}, A}\}\big).$

When $\min\{{d^{-1}\log{S}, A}\} > \varepsilon^{-2}$, select $K = \tilde{O}\big(H^4d\varepsilon^{-2}\big)$, $\tau = \tilde{O}\big(H^{-4}\varepsilon^{2}d^{-1}\big)$, then it holds
\begin{align*}
   H^2 \sqrt{\tau}\big(\sqrt{\frac{ \log(S) }{K }} \wedge \sqrt{ d(\frac{A}{K} \wedge 1 )} + \gamma H^2 \sqrt{\frac{2\log A_i}{K}} 
   & \le H^2\sqrt{\tau d} + \gamma H^2 \sqrt{\frac{2\log A_i}{K}}\\
   & \lesssim H^2\sqrt{\tau d} + H^2 \sqrt{\frac{d}{K}} \\
   & \lesssim \varepsilon.
\end{align*}
And the corresponding query of samples is $m^2H^2KC_{\max}^2 = \tilde{O}\big(m^2H^{14}d^5\varepsilon^{-6}\big).$

Thus it is sufficient to prove \eqref{eq-epoch-prob} for every fixed $\ell$. For simplicity of the notation, we omit the index $\ell$ in the followed analysis.

\subsubsection{Proof of~\eqref{eq-epoch-prob}}
We recall the following notations:
\begin{align*}
{V}^{\tilde\pi}_{h,i} &= \mathbb{E}_{\bm a \sim \tilde\pi}[Q^{\tilde\pi}_{h,i}(s,\bm a)],\\
V_{h,i}^{\dagger, \tilde{\pi}_{-i}}(s) &= \max_{a}\big\{ r_{h,i}^{\tilde{\pi}_{h,-i}}(s,a) + \mathbb{P}_{h}^{\tilde{\pi}_{-i}}V_{h+1,i}^{\dagger, \tilde{\pi}_{-i}}(s,a)\big\}, \quad \text{with} ~V_{H+1,i}^{\dagger, \tilde{\pi}_{-i}} = 0\\
P_{h}^{\tilde{\pi}_{-i}}V(s,a) &=\E_{\bm a_{-i}\sim \tilde{\pi}_{-i}}[ \E_{s' \sim \mathbb{P}_h(\cdot \lvert s,a,\bm{a}_{-i})}[V(s')]]\\
r_{h,i}^{\tilde{\pi}_{h,-i}}(s,a) &= \mathbb{E}_{\bm{a}_{-i}\sim \tilde\pi_{h,-i}(\cdot \lvert s)}\big[r_{h,i}(s,a,\bm{a}_{-i}) \big],\\
\tilde{V}_{h,i}(s) &= \min\left\{\frac{1}{K}\sum_{k=1}^K \E_{\bm a\sim \tilde{\pi}^k_{h,i}}[\tilde{Q}_{h,i}^k(s,\bm a)], H-h+1 \right\}, \\
Q^{\tilde\pi_{h,-i}^k,\tilde{V}_{h+1,i}}_{h,i}(s,a) &=\mathbb{E}_{\bm{a}_{-i}\sim \tilde\pi^k_{h,-i}(\cdot \lvert s)}\big[r_{h,i}(s,a,\bm{a}_{-i})+\mathbb{E}_{s^{\prime} \sim \mathbb{P}_h(\cdot \lvert s,a,\bm{a}_{-i})}[\tilde{V}_{h+1,i}(s^{\prime})]\big] \\
\end{align*}

We begin with the following decomposition for each $h\in [H]$:
  \begin{align*}
    {V}^{\dagger, \tilde\pi_{-i}}_{h,i} - {V}^{\tilde\pi}_{h,i} 
    &= {V}^{\dagger,\tilde\pi_{-i}}_{h,i} - {V}^{\tilde\pi}_{h,i} \pm \tilde{V}_{h,i} \\
    &= \underbrace{{V}^{\dagger, \tilde\pi_{-i}}_{h,i} - \tilde{V}_{h,i}}_{I_{h,i}^{(1)}} + \underbrace{\tilde{V}_{h,i} - {V}^{\tilde\pi}_{h,i}}_{I_{h,i}^{(2)}}.
 \end{align*}

We deal with $I_{h,i}^{(1)}$ and $I_{h,i}^{(2)}$ separately in the following two subsections:

\begin{lemma}\label{lem-I1-bound}
  For every $i\in [m], h \in [H]$, we have, with high probability, \begin{equation}\label{eq-I1-bound}
    I_{1,i}^{(1)}\lesssim  H^2\sqrt{\frac{\tau \log(1/\delta)}{K}\min\{d,\log S\} } + 2H\nu \sqrt{d} + \gamma H^2 \sqrt{\frac{2\log A_i}{K}}.
  \end{equation}
  \end{lemma}

\begin{lemma}\label{lem-I2-bound}
  For every $i\in [m], h \in [H]$, we have, with high probability, \begin{equation}\label{eq-I2-bound_conclusion}
    I_{1,i}^{(2)}\lesssim H\nu \sqrt{d}+ H^2 \sqrt{\tau}\big(\sqrt{\frac{ \log(S) }{K }} \wedge \sqrt{ d(\frac{A}{K} \wedge 1 )}\big).
  \end{equation}

\end{lemma}

\subsubsection{Proof of Lemma~\ref{lem-I1-bound}}\label{sec-proof-lem2}
\begin{proof}
We would bound $I_{h,i}^{(1)} $ by taking backward induction on $h$:

Firstly we have $I_{H+1,i}^{(1)} = 0$ by definition, now suppose it holds for $h+1$ that with probability at least $1-\delta_{h+1}$,\begin{align*}
  I_{h+1,i}^{(1)} \leq z_{h+1, i}  
\end{align*}
for some $z_{h+1,i} \geq 0,$

then we have \begin{align*}
  V_{h,i}^{\dagger, \tilde{\pi}_{-i}}(s) &= \max_{a}\big\{ r_{h,i}^{\tilde{\pi}_{h,-i}}(s,a) + \mathbb{P}_{h}^{\pi_{-i}}V_{h+1,i}^{\dagger, \tilde{\pi}_{-i}}(s,a)\big\}\\
  &\leq \max_{a}\big\{ r_{h,i}^{\tilde{\pi}_{h,-i}}(s,a) +  \mathbb{P}_{h}^{\pi_{-i}}\tilde{V}_{h+1,i}(s,a)\big\} + z_{h+1,i}
\end{align*}
Now if we denote \begin{align*}
  \zeta_{h,i}(s,a):=\big\lvert r_{h,i}^{\tilde{\pi}_{h,-i}}(s,a) + \mathbb{P}_{h}\tilde{V}_{h+1,i}(s,a) - \frac{1}{K}\sum_{k=1}^K \tilde{Q}_{h,i}^k(s,a)\big\rvert,
\end{align*}
then it holds by induction assumption that with probability at least $1-\delta_{h+1}$ \begin{align*}
V_{h,i}^{\dagger, \tilde{\pi}_{-i}}(s) &\leq  \max_{a}\frac{1}{K}\sum_{k=1}^{K} \tilde{Q}_{h,i}^k(s,a) + \max_a \zeta_{h,i}(s,a) + z_{h+1,i}\\
  &\leq \tilde{V}_{h,i}(s) + \textit{Reg}(\textit{FTRL}) + \max_a \zeta_{h,i}(s,a) + z_{h+1,i},
\end{align*}
with \begin{align*}
  \textit{Reg}(\textit{FTRL}) : =  \frac{1}{K} \max_{a'}\big(\sum_{k=1}^K \tilde{Q}_{h,i}^k(s,a') - \sum_{k=1}^K \E_{a\sim \tilde{\pi}^k_{h,i}}[\tilde{Q}_{h,i}^k(s,a)]\big)
\end{align*}

For $\textit{Reg}(\textit{FTRL})$ we have the following lemma:
\begin{lemma}\label{lem-ftrl-our-bound}
  For $\gamma: = \min\{1+\sqrt{\tau \log(SA/\delta)}+ \nu \sqrt{d}, \sqrt{d}\}$, we have with probability at least $1-\delta$, it holds for any $s\in \mathcal{S}$ that \begin{align*}
    \textit{Reg}(\textit{FTRL})\lesssim \gamma H \sqrt{\frac{2\log A_i}{K}}.
  \end{align*}
\end{lemma}

Now it remains to bound $\zeta_{h,i}$: 

When $s\in \mathcal{C}_h,$ we have \begin{align*}
  &\quad \frac{1}{K}\sum_{k=1}^K \tilde{Q}_{h,i}^k(s,a)\\
   &= \frac{1}{K}\sum_{k=1}^K  \sum_{(\tilde{s},\tilde{a}) \in \mathcal{D}_{h,i} }\phi_i(s,a)^\top \Lambda_{h,i}^{-1} \phi_i(\tilde{s},\tilde{a})q_{h,i}^k(\tilde{s},\tilde{a})\\
  &= \frac{1}{K} \sum_{(\tilde{s},\tilde{a}) \in \mathcal{D}_{h,i} }\phi_i(s,a)^\top \Lambda_{h,i}^{-1} \phi_i(\tilde{s},\tilde{a})\big[\sum_{k=1}^K  Q_{h,i}^{\tilde\pi_{h,-i}^k,\tilde{V}_{h+1,i}}(\tilde{s},\tilde{a}) )\big] + \underbrace{J_1}_{\text{Martingale Concentration}}\\
  &= \frac{1}{K} \sum_{(\tilde{s},\tilde{a}) \in \mathcal{D}_{h,i} }\phi_i(s,a)^\top \Lambda_{h,i}^{-1} \phi_i(\tilde{s},\tilde{a})\phi_i(\tilde{s},\tilde{a})^\top \sum_{k=1}^K \theta_{h,i}^{\tilde\pi_{h,-i}^k,\tilde{V}_{h+1,i}} + \underbrace{J_1}_{\text{Martingale Concentration}} +\underbrace{J_2}_{\text{Misspecfic Error}}\\
  &= \frac{1}{K}\sum_{k=1}^K \phi_i(s,a)^\top \theta_{h,i}^{\tilde\pi_{h,-i}^k,\tilde{V}_{h+1,i}} + \underbrace{J_1}_{\text{Martingale Concentration}} +\underbrace{J_2}_{\text{Misspecfic Error}} +\underbrace{J_3}_{\text{Incurred by $\lambda$}}\\
  &= \frac{1}{K} \sum_{k=1}^K Q_{h,i}^{\tilde\pi_{h,-i}^k,\tilde{V}_{h+1,i}}(s,a)  +O(\nu) + \underbrace{J_1}_{\text{Martingale Concentration}} +\underbrace{J_2}_{\text{Misspecfic Error}} +\underbrace{J_3}_{\text{Incurred by $\lambda$}}
\end{align*}
where\begin{align*}
  J_1 &=\frac{1}{K}\sum_{(\tilde{s},\tilde{a}) \in \mathcal{D}_{h,i} }\phi_i(s,a)^\top \Lambda_{h,i}^{-1} \phi_i(\tilde{s},\tilde{a})\big[\sum_{k=1}^K ({q}_{h,i}^k(\tilde{s},\tilde{a}) - Q_{h,i}^{\tilde\pi_{h,-i}^k,\tilde{V}_{h+1,i}}(\tilde{s},\tilde{a}) ))\big]\\
  J_2 &= \frac{1}{K} \sum_{(\tilde{s},\tilde{a}) \in \mathcal{D}_{h,i} }\phi_i(s,a)^\top \Lambda_{h,i}^{-1} \phi_i(\tilde{s},\tilde{a})\big[\sum_{k=1}^K ( Q_{h,i}^{\tilde\pi_{h,-i}^k,\tilde{V}_{h+1,i}}(\tilde{s},\tilde{a}) ) -\phi_i(\tilde s,\tilde a)^\top \theta_{h,i}^{\tilde\pi_{h,-i}^k,\tilde{V}_{h+1,i}} )\big]  \\
  J_3 &= -\frac{1}{K}\sum_{k=1}^K \lambda\phi_i(s,a)^\top \Lambda_{h,i}^{-1}  \theta_{h,i}^{\tilde\pi_{h,-i}^k,\tilde{V}_{h+1,i}},
\end{align*}
with \begin{align*}
  \theta_{h,i}^{\tilde\pi_{h,-i}^k,\tilde{V}_{h+1,i}} = \text{argmin}_{\lVert \theta\rVert_2\leq H\sqrt{d}} \lVert  Q_{h,i}^{\tilde\pi_{h,-i}^k,\tilde{V}_{h+1,i}} - \phi_i(\cdot,\cdot)^\top\theta \rVert_\infty.
\end{align*}

Now we aim to control $J_1,J_2,J_3$ separately:
   \paragraph{Bounding $J_1$} For $J_1$, we have the following Lemma:\begin{lemma}\label{lem-J1-bound} With probability at least $1-\delta,$ it holds that \begin{align*}
    J_1 \lesssim H\sqrt{\frac{\tau \log(1/\delta)}{K}\min\{d,\log S\} }
.
  \end{align*}
  \end{lemma}
  \paragraph{Bounding $J_2$} For $J_2$, we have by the Assumption~\ref{assumption-MDP},
  \begin{align*}
  \lvert J_2 \rvert &\leq \lvert \phi_i(s,a)^\top \sum_{\tilde{s},\tilde{a}\in \mathcal{D}_{h,i}}\Lambda_{h,i}^{-1}\phi_i(\tilde{s},\tilde{a}) \cdot \nu \rvert  \\
  &\leq \sqrt{ \lvert \mathcal{D}_{h,i} \rvert  \sum_{(\tilde{s},\tilde{a}) \in \mathcal{D}_{h,i} }  \lvert \phi_i(s,a)^\top \Lambda_{h,i}^{-1}\phi_i(\tilde{s},\tilde{a}) \rvert^2}  \cdot \nu \\
  &\lesssim \nu \sqrt{d}.
  \end{align*}
  \paragraph{Bounding $J_3$}   For $J_3$, we have it holds straightforwardly that 
  \begin{align*}
  \lvert J_3 \rvert &\leq  H \sqrt{\tau \lambda d}
  \end{align*}

In addition, noticing that for $s\notin \mathcal{C}_h,$ by our design of virtual algorithm , it holds that \begin{align*}
  \sqrt{\lvert \tilde{D}_{h,i}\rvert} \lesssim \sqrt{\frac{d}{\tau}} ,\quad \lVert \phi_i(s,a)\rVert_{\tilde{\Lambda}_{h,i}^{-1}}^2 \leq \tau, \forall s\notin \mathcal{C}_{h}
\end{align*}  
thus our arguement when $s\in \mathcal{C}_h$(including the proof of Lemma~\ref{lem-J1-bound}) still holds by replacing $\mathcal{D}_{h,i},{\Lambda}_h^{-1}$ by $\tilde{\mathcal{D}}_{h,i},\tilde{\Lambda}_h^{-1}$. 

Finally, selecting $\lambda = \lambda_0:=1/{KdH^2}$, and by the induction assumption on $(h+1)$-th step, we have with probability at least $1-(\delta_{h+1}+\delta)$,
\begin{align}\label{eq-I1-recursive}
  z_h = z_{h+1} + O(H\sqrt{\frac{\tau \log(1/\delta)}{K}\min\{d,\log S\} } + \gamma H \sqrt{\frac{2\log A_i}{K}}) + 2\nu \sqrt{d}
\end{align}

Thus \eqref{eq-I1-recursive} recursively with $z_{H+1} = \delta_{H+1} = 0$ leads to with probability at least $1-H\delta$,
\begin{align*}
  I_{1,i}^{(1)}\lesssim  H^2\sqrt{\frac{\tau \log(1/\delta)}{K}\min\{d,\log S\} } + \gamma H^2 \sqrt{\frac{2\log A_i}{K}} + 2H\nu \sqrt{d}
\end{align*}

\end{proof}

\subsubsection{Proof of Lemma~\ref{lem-I2-bound}}
We would also bound $I^{(2)}_{h,i}$ via backward induction on $h$:

Firstly, we have $I^{(2)}_{H+1,i} = 0$ by definition, now suppose it holds for $h+1$ that with probability at least $1-\delta_{h+1},$\begin{align*}
  \lVert I^{(2)}_{h+1,i}  \rVert_\infty \leq \xi_{h+1},
\end{align*}
then we have for $h$-th time-step, for every $s\in \mathcal{C}_{h}$, 
\begin{align*}
  &\quad \frac{1}{K}\sum_{k=1}^K \E_{a_i\sim\tilde{\pi}^{k}_{h,i}}\big[\tilde{Q}_{h,i}^k(s,a_i) \big]\\
  &= \frac{1}{K}\sum_{k=1}^{K} \E_{a_i \sim \tilde{\pi}_{h,i}^k}\bigg[  \big\langle \phi_i(s,a_i), \Lambda_{h,i}^{-1}\sum_{(\tilde s,\tilde a)\in \mathcal{D}_{h,i}}\phi_i(\tilde{s},\tilde{a})\big( {r}^k_{i,h}(\tilde{s},\tilde{a}) + \tilde{V}_{h+1,i}(s^k_{\tilde{s},\tilde{a}})  \big)    \big\rangle  \bigg]\\
  &= \underbrace{\frac{1}{K}\sum_{k=1}^{K}  \E_{a_i \sim \tilde{\pi}_{h,i}^k}\bigg[  \phi_i(s,a_i)^\top\Lambda_{h,i}^{-1}\sum_{(\tilde s,\tilde a)\in \mathcal{D}_{h,i}}\phi_i(\tilde{s},\tilde{a}) Q_{h,i}^{\tilde{\pi}_{h,-i}^k,\tilde{V}_{h+1,i}}(\tilde{s},\tilde{a})   \big)      \bigg]}_{G_{1}} \\
  &\quad+ \underbrace{\frac{1}{K}\sum_{k=1}^{K}  \E_{a_i \sim \tilde\pi_{h,i}^k}\bigg[  \phi_i(s,a_i)^\top \Lambda_{h,i}^{-1}\sum_{(\tilde s,\tilde a)\in \mathcal{D}_{h,i}}\phi_i(\tilde{s},\tilde{a}) \varepsilon_k(\tilde{s},\tilde{a})   \bigg]}_{\text{$G_{2}$}}.
\end{align*}
with 
\begin{align*}
  \varepsilon_k(\tilde{s},\tilde{a}):= \big( {r}^k_{i,h}(\tilde{s},\tilde{a}) + \tilde{V}_{h+1,i}(s^k_{\tilde{s},\tilde{a}})  \big)  - Q_{h,i}^{\tilde{\pi}_{h,-i}^k,\tilde{V}_{h+1,i}}(\tilde{s},\tilde{a})  \big).
\end{align*}
where ${r}^k_{i,h}(\tilde{s},\tilde{a}), s^k_{\tilde{s},\tilde{a}}$ denote the $r, s^\prime$ obtained by local sampling$(h,i, \tilde{s}, \tilde{a}, \tilde\pi^k_{h,-i})$.

Now we would discuss $G_1,G_2$ separately:

\textbf{Bounding $G_1$:}\\
By Assumption~\ref{assumption-MDP} and the induction assumption on $h+1$, we have  
\begin{align*}
  G_1  &= \frac{1}{K}\sum_{k=1}^{K}  \E_{a_i \sim \tilde{\pi}_{h,i}^k}\bigg[  \phi_i(s,a_i)^\top\Lambda_{h,i}^{-1}\sum_{(\tilde s,\tilde a)\in \mathcal{D}_{h,i}}\phi_i(\tilde{s},\tilde{a}) Q_{h,i}^{\tilde{\pi}_{h,-i}^k,\tilde{V}_{h+1,i}}(\tilde{s},\tilde{a})   \big)      \bigg] \\
  &=  \frac{1}{K}\sum_{k=1}^K \E_{a_i \sim\tilde{\pi}_{h,i}^k}[ Q_{h,i}^{\tilde\pi_{h,-i}^k,\tilde{V}_{h+1,i}}(s,a_i) ]   +\tilde{O}( \nu \sqrt{d}) .
\end{align*}
Where in the last equation we used the similar argument as in bounding $J_2,J_3$ in previous section.

Now noticing that with probability at least $1-\delta_{h+1},$
\begin{align*}
  \lvert Q_{h,i}^{\tilde\pi_{h,-i}^k,\tilde{V}_{h+1,i}}(s,a) - 
 Q_{h,i}^{\tilde\pi_{h,-i}^k,{V}^{\tilde{\pi}}_{h+1,i}}(s,a)\rvert   &\leq  \int \lvert  \tilde{V}_{h+1,i}(s') - {V}^{\tilde{\pi}}_{h+1,i}(s') \rvert d\mathbb{P}_{h}^{\tilde\pi_{h,-i}^k}(s' \lvert s,a)\\
 &\leq \xi_{h+1},
\end{align*}
we have with probability at least $1-\delta_{h+1},$ \begin{align*}
 \lvert G_{1} - V_{h,i}^{\tilde{\pi}}(s) \rvert &\leq   \frac{1}{K}\sum_{k=1}^K \E_{a_i \sim\tilde{\pi}_{h,i}^k}[ \lvert  Q_{h,i}^{\tilde\pi_{h,-i}^k,\tilde{V}_{h+1,i}}(s,a_i) -   Q_{h,i}^{\tilde\pi_{h,-i}^k,{V}^{\tilde{\pi}}_{h+1,i}}(s,a_i)\rvert  ] +\tilde{O}( \nu \sqrt{d}) \\
 &\leq \xi_{h+1} +\tilde{O}( \nu \sqrt{d}).
\end{align*}


\textbf{Bounding $G_{2}$:}

We have following lemma regarding the upper bound of $\lvert G_2 \rvert:$
\begin{lemma}\label{lem-G2-bound}
With probability at least $1-\delta$, it holds that \begin{align*}
  \lvert G_2 \rvert \lesssim H\sqrt{\tau\min\bigg\{\frac{\log(S/\delta)}{K}, {d}\big(1 \wedge {A}/{K}\big)\cdot \bigg\}}.
\end{align*}
\end{lemma}

\begin{proof}[Proof of Lemma~\ref{lem-G2-bound}]
For any fixed $\tilde{s} \in \mathcal{C}_h$ we have denoted $\mathcal{F}_k(\tilde{s},\tilde{a})$ the filtration generated by the information before taking the $k$-th sampling on $\tilde{s},\tilde{a}$, then for \begin{align*}
  \tilde{Z}_{k}(\tilde{s},\tilde{a}) := \E_{a\sim \tilde\pi^{k}_{h,i}}[\phi_i(s,a)]^\top \Lambda_{h,i}^{-1}\phi_i(\tilde{s},\tilde{a}) q_{h,i}^k(\tilde{s},\tilde{a}),
\end{align*}
it holds that 
$$ \E[\tilde{Z}_{k}(\tilde{s},\tilde{a})\lvert \mathcal{F}_{k}(\tilde{s},\tilde{a})] =  \E_{a\sim \tilde\pi_{h,i}^k}[\phi_i(s,a)]^\top \Lambda_{h,i}^{-1}\phi_i(\tilde{s},\tilde{a})\big( r_{h,i}^{\tilde\pi^k_{h,-i}}(\tilde{s},\tilde{a}) + \mathbb{P}_{h}^{\tilde\pi_{h,-i}^{k}}\tilde{V}_{h+1,i}(\tilde{s},\tilde{a}) \big),$$
and $\lvert \tilde{Z}_k(\tilde{s},\tilde{a}) \rvert \leq H \max_{a}\lvert \phi_i(s,a)^\top \Lambda_{h,i}^{-1}\phi_i(\tilde{s},\tilde{a}) \rvert\leq H\tau $ almost surely. Moreover, notice that \begin{align*}
  \text{Var}[\tilde{Z}_{k}(\tilde{s},\tilde{a})\lvert \mathcal{F}_{k}(\tilde{s},\tilde{a})] &\leq \mathbb{E}[\tilde{Z}_{k}(\tilde{s},\tilde{a})^2\lvert \mathcal{F}_{k}(\tilde{s},\tilde{a})]\leq H^2 \lvert \E_{a\sim \tilde\pi_{h,i}^k}[\phi_i(s,a)]^\top \Lambda_{h,i}^{-1}\phi_i(\tilde{s},\tilde{a}) \big\rvert^2,
\end{align*}

thus applying Freedman's inequality as in \cite{li2020generative} leads to with probability at least $1-\delta$,
\begin{align*}
  &\frac{1}{K}\sum_{(\tilde{s},\tilde{a}) \in \mathcal{D}_{h,i} }\sum_{k=1}^K  \E_{a\sim \tilde\pi_{h,i}^k}[\phi_i(s,a)]^\top \Lambda_{h,i}^{-1} \phi_i(\tilde{s},\tilde{a})\big[(q_{h,i}^k(\tilde{s},\tilde{a}) - r_{h,i}^{\tilde\pi^k_{h,-i}}(\tilde{s},\tilde{a}) - \mathbb{P}_{h}^{\tilde\pi_{h,-i}^{k}}\tilde{V}_{h+1,i}(\tilde{s},\tilde{a}))\big]\\
=&\frac{1}{K}\sum_{(\tilde{s},\tilde{a}) \in \mathcal{D}_{h,i} }\sum_{k=1}^K \big( \tilde{Z}_{k}(\tilde{s},\tilde{a}) - \E[\tilde{Z}_k(\tilde{s},\tilde{a})\lvert \mathcal{F}_{k}(\tilde{s},\tilde{a})] \big)\\
\lesssim & \frac{H}{K}\bigg(\sqrt{\sum_{k=1}^K \sum_{(\tilde{s},\tilde{a}) \in \mathcal{D}_{h,i}} \lvert \E_{a\sim \tilde\pi_{h,i}^k}[\phi_i(s,a)]^\top \Lambda_{h,i}^{-1}\phi_i(\tilde{s},\tilde{a}) \big\rvert^2 \log(Kd/\delta) } + \tau \log(Kd/\delta) \bigg)\\
\lesssim & H\sqrt{\frac{\tau}{K}\log(Kd/\delta)} + \frac{H\tau}{K} \log(Kd/\delta)
\end{align*} 
the last line is by \begin{align*}
  &\sum_{k=1}^K \sum_{(\tilde{s},\tilde{a}) \in \mathcal{D}_{h,i}} \lvert \E_{a\sim \tilde\pi_{h,i}^k}[\phi_i(s,a)]^\top \Lambda_{h,i}^{-1}\phi_i(\tilde{s},\tilde{a}) \big\rvert^2\\
   =&  \sum_{k=1}^K \sum_{(\tilde{s},\tilde{a}) \in \mathcal{D}_{h,i}}  \E_{a\sim \tilde\pi_{h,i}^k}[\phi_i(s,a)]^\top \Lambda_{h,i}^{-1}\phi_i(\tilde{s},\tilde{a}) \phi_i(\tilde{s},\tilde{a})^\top \Lambda_{h,i}^{-1}\E_{a\sim \tilde\pi_{h,i}^k}[\phi_i(s,a)] \\
   \leq &  \sum_{k=1}^K  \E_{a\sim \tilde\pi_{h,i}^k}[\phi_i(s,a)]^\top \Lambda_{h,i}^{-1}\E_{a\sim \tilde\pi_{h,i}^k}[\phi_i(s,a)] \\
   \leq &K\tau
\end{align*}
Now taking union bound over all $s\in \mathcal{C}_h $, we get
with probability at least $1-\delta$, \begin{align}\label{eq-G2-bound-SA}
 \lvert G_2 \rvert  & \lesssim H\sqrt{\frac{\tau \log(S/\delta) }{K}}.
\end{align}

On the other hand,  we have \begin{align*}
  &\big\lvert \frac{1}{K}\sum_{(\tilde{s},\tilde{a}) \in \mathcal{D}_{h,i} }\sum_{k=1}^K  \E_{a\sim \tilde\pi_{h,i}^k}[\phi_i(s,a)]^\top \Lambda_{h,i}^{-1} \phi_i(\tilde{s},\tilde{a})\big[({q}_{h,i}^k(\tilde{s},\tilde{a}) - r_{h,i}^{\tilde\pi^k_{h,-i}}(\tilde{s},\tilde{a}) - \mathbb{P}_{h}^{\tilde\pi_{h,-i}^{k}}\tilde{V}_{h+1,i}(\tilde{s},\tilde{a}))\big]\big\rvert\\
  \leq &\frac{\sqrt{\tau}}{K}\sum_{k=1}^K {\big\lVert \sum_{(\tilde{s},\tilde{a})\in \mathcal{D}_{h,i}} \phi_i(\tilde{s},\tilde{a})\big[({q}_{h,i}^k(\tilde{s},\tilde{a}) - r_{h,i}^{\tilde\pi^k_{h,-i}}(\tilde{s},\tilde{a}) - \mathbb{P}_{h}^{\tilde\pi_{h,-i}^{k}}\tilde{V}_{h+1,i}(\tilde{s},\tilde{a}))\big] \big\rVert_{\Lambda^{-1}}}.
\end{align*}
Now for each $k$, consider the $\epsilon_0$-net $\mathcal{N}_\epsilon$ of $\mathbb{B}_{d},$ then it holds that with probability at least $1-\delta,$
\begin{align*}
  &{\big\lVert \sum_{(\tilde{s},\tilde{a})\in \mathcal{D}_{h,i}} \phi_i(\tilde{s},\tilde{a})\big[({q}_{h,i}^k(\tilde{s},\tilde{a}) - r_{h,i}^{\pi^k_{h,-i}}(\tilde{s},\tilde{a}) - \mathbb{P}_{h}^{\pi_{h,-i}^{k}}\tilde{V}_{h+1,i}(\tilde{s},\tilde{a}))\big] \big\rVert_{\Lambda^{-1}}}\\
  \leq &\sup_{g\in \mathcal{N}_\epsilon}  g^\top \Lambda^{-1/2}\sum_{(\tilde{s},\tilde{a})\in \mathcal{D}_{h,i}} \phi_i(\tilde{s},\tilde{a})\big[({q}_{h,i}^k(\tilde{s},\tilde{a}) - r_{h,i}^{\pi^k_{h,-i}}(\tilde{s},\tilde{a}) - \mathbb{P}_{h}^{\pi_{h,-i}^{k}}\tilde{V}_{h+1,i}(\tilde{s},\tilde{a}))\big] + \epsilon_0 H \sqrt{\tau \lvert \mathcal{D}_{h,i}\rvert}\\
  \leq &H\sqrt{\log(\lvert \mathcal{N}_{\epsilon_0}\rvert/\delta)} + \epsilon_0 H \sqrt{\tau \lvert \mathcal{D}_{h,i}\rvert}.
\end{align*}\
Now letting $\epsilon_0 = O(\sqrt{\frac{1}{K \tau\lvert \mathcal{D}_{h,i}\rvert}})$ and noticing that $\log\lvert \mathcal{N}_{\epsilon_0}\rvert = \tilde{O}(d)$, we have it holds with probability at least $1-\delta$ that\begin{equation}\label{eq-G2-tau}
  \lvert G_2 \rvert \lesssim H\sqrt{\tau\big[\log(\lvert \mathcal{N}_{\epsilon_0}\rvert/\delta) + \frac{1}{K}\big] } = \tilde{O}(H\sqrt{\tau d})
\end{equation}

Finally, if we consider the metric over $\mathcal{S}$: 
\begin{align*}
  D(s,s') := \max_{a}\lVert \phi(s,a) - \phi(s',a)\rVert
\end{align*}
then if we consider the minimal $\epsilon$ cover $\mathcal{N}_\epsilon$ of $\mathcal{F}$, it holds trivially that \begin{equation}
  \lvert \mathcal{N}_D(\mathcal{S};\epsilon) \rvert \leq \lvert \mathcal{N}_\epsilon \rvert^A 
\end{equation}
by $\{\phi(s,a)_{a\in \mathcal{A}}: s\in \mathcal{S}\} \subset \mathcal{F}^{A}$ and the fact $\big\lvert \mathcal{N}_{\lVert \cdot \rVert_{2,\infty}}(\mathcal{F}^A;\epsilon)\big\rvert \leq  \lvert \mathcal{N}_\epsilon\rvert^A.$ 

Now for any $s \in \mathcal{C}_h,$ consider its best approximation $\bar{s} \in \mathcal{N}_D(\mathcal{S};\epsilon),$ then it holds that {\footnotesize \begin{equation}\label{eq-A-cover}
\begin{aligned}
  &\big\lvert \frac{1}{K}\sum_{(\tilde{s},\tilde{a}) \in \mathcal{D}_{h,i} }\sum_{k=1}^K  \E_{a\sim \tilde\pi_{h,i}^k}[\phi_i(s,a)]^\top \Lambda_{h,i}^{-1} \phi_i(\tilde{s},\tilde{a})\big[({q}_{h,i}^k(\tilde{s},\tilde{a}) - r_{h,i}^{\tilde\pi^k_{h,-i}}(\tilde{s},\tilde{a}) - \mathbb{P}_{h}^{\tilde\pi_{h,-i}^{k}}\tilde{V}_{h+1,i}(\tilde{s},\tilde{a}))\big]\big\rvert\\
  \leq &\big\lvert \frac{1}{K}\sum_{(\tilde{s},\tilde{a}) \in \mathcal{D}_{h,i} }\sum_{k=1}^K  \E_{a\sim \tilde\pi_{h,i}^k}[\phi_i(\bar{s},a)]^\top \Lambda_{h,i}^{-1} \phi_i(\tilde{s},\tilde{a})\big[({q}_{h,i}^k(\tilde{s},\tilde{a}) - r_{h,i}^{\tilde\pi^k_{h,-i}}(\tilde{s},\tilde{a}) - \mathbb{P}_{h}^{\tilde\pi_{h,-i}^{k}}\tilde{V}_{h+1,i}(\tilde{s},\tilde{a}))\big]\big\rvert \\
   +& \big\lvert \frac{1}{K}\sum_{k=1}^K (\E_{a\sim \tilde\pi_{h,i}^k}[\phi_i(\bar{s},a)] - \E_{a\sim \tilde\pi_{h,i}^k}[\phi_i({s},a)])\Lambda^{-1} \sum_{(\tilde{s},\tilde{a})\in \mathcal{D}_{h,i}} \phi_i(\tilde{s},\tilde{a})\big[({q}_{h,i}^k(\tilde{s},\tilde{a}) - r_{h,i}^{\tilde\pi^k_{h,-i}}(\tilde{s},\tilde{a}) - \mathbb{P}_{h}^{\tilde\pi_{h,-i}^{k}}\tilde{V}_{h+1,i}(\tilde{s},\tilde{a}))\big]\big\rvert.
\end{aligned}\end{equation}}

In particular, noticing that by $D(s,\bar{s}) \leq \epsilon,$ we have \begin{align*}
    &\E_{a\sim \tilde\pi_{h,i}^k}[\phi_i({s},a)]\\
=&\sum_{a}\tilde\pi_{h,i}^k(a\lvert s ) \phi_i(s,a)\\
=&\sum_{a}\tilde\pi_{h,i}^k(a\lvert \bar{s} )  \phi_i(\bar{s},a) +  [\sum_{a}\tilde\pi_{h,i}^k(a\lvert s )- \sum_{a}\tilde\pi_{h,i}^k(a\lvert \bar{s} )]  \phi_i(s,a) + \sum_{a}\tilde\pi_{h,i}^k(a\lvert \bar{s} ) [\phi_i(\bar{s},a) - \phi_i({s},a)].
\end{align*}

Now by \begin{align*}
   \lVert  \sum_{a}\tilde\pi_{h,i}^k(a\lvert \bar{s} ) [\phi_i(\bar{s},a) - \phi_i({s},a)]\rVert   &\leq D(s,\bar{s})  \leq \epsilon
\end{align*}
and 
\begin{align*}
 &\lVert [\sum_{a}\tilde\pi_{h,i}^k(a\lvert s )- \sum_{a}\tilde\pi_{h,i}^k(a\lvert \bar{s} )]  \phi_i(s,a) \rVert\\
 \leq &\max_a\lVert \phi_i(s,a) \rVert \cdot\sum_{a}\lvert \tilde\pi_{h,i}^k(a\lvert s ) - \tilde\pi_{h,i}^k(a\lvert \bar{s} )\rvert\\
 \leq & \lVert  \text{SoftMax}\big(\bar{Q}_{h,i}^k(\cdot,s),\eta_k)  -  \text{SoftMax}\big(\bar{Q}_{h,i}^k(\cdot,\bar{s}),\eta_k\big) \rVert_1 \\
 \leq &\eta_k \lVert \bar{Q}_{h,i}^k(\cdot,s) - 
    \bar{Q}_{h,i}^k(\cdot,\bar{s})\rVert_\infty, 
\end{align*}
where in the last line we have used the explicit formula of Jacobian of Soft-Max function in \cite{gao2017properties} and the following inequality: \begin{align*}
    \lVert \sigma_\eta\big(z) - \sigma_\eta\big(z') \rVert_1 &= \lVert \langle D\sigma_{\lambda}(\xi)(z-z') \rVert_1\\
    &\leq \lVert z - z' \rVert_\infty   \sum_{i,j}\lvert D\sigma_\lambda(\xi) \rvert_{i,j}\\
    &\leq \eta \lVert z - z' \rVert_\infty.
\end{align*}
and \begin{align*}
   \lVert \bar{Q}_{h,i}^k(\cdot,s) - 
    \bar{Q}_{h,i}^k(\cdot,\bar{s})\rVert_\infty  &\leq \max_k\lVert Q_{h,i}^k(\cdot,s) - Q_{h,i}^k(\cdot,\bar{s})\rVert_\infty\\
    &\leq \max_{a}\lVert \phi_{i}(s,a) - \phi_{i}(\bar{s},a) \rVert_2 \lVert \hat{\theta} \rVert_2 \\
    &\leq D(s,\bar{s}) O(Hd/\lambda)
\end{align*}
we get 
 $$ \lVert \E_{a\sim \tilde\pi_{h,i}^k}[\phi_i(\bar{s},a)] - \E_{a\sim \tilde\pi_{h,i}^k}[\phi_i({s},a)]\rVert_{\Lambda^{-1}} \lesssim \epsilon\sqrt{K/\lambda}.$$
with proper choice of $\epsilon$.

Applying Freedman's inequality in the last line of \eqref{eq-A-cover}, we have then with probability at least $1-\delta$,  
\begin{align*}
    &\big\lvert \frac{1}{K}\sum_{(\tilde{s},\tilde{a}) \in \mathcal{D}_{h,i} }\sum_{k=1}^K  \E_{a\sim \tilde\pi_{h,i}^k}[\phi_i(\bar{s},a)]^\top \Lambda_{h,i}^{-1} \phi_i(\tilde{s},\tilde{a})\big[({q}_{h,i}^k(\tilde{s},\tilde{a}) - r_{h,i}^{\tilde\pi^k_{h,-i}}(\tilde{s},\tilde{a}) - \mathbb{P}_{h}^{\tilde\pi_{h,-i}^{k}}\tilde{V}_{h+1,i}(\tilde{s},\tilde{a}))\big]\big\rvert \\
   +& O\big(\epsilon H \sqrt{\tau /\lambda} \big).
\end{align*}
Thus it suffice to control the deviation over the first term with fixed $\bar{s} \in \mathcal{N}_\epsilon^D$. Using exactly the same arguement as in establishing \eqref{eq-G2-bound-SA}, we have for every $\bar{s},$ it holds that \begin{align*}
  \big\lvert \frac{1}{K}\sum_{(\tilde{s},\tilde{a}) \in \mathcal{D}_{h,i} }\sum_{k=1}^K  \E_{a\sim \tilde\pi_{h,i}^k}[\phi_i(\bar{s},a)]^\top \Lambda_{h,i}^{-1} \phi_i(\tilde{s},\tilde{a})\big[({q}_{h,i}^k(\tilde{s},\tilde{a}) - r_{h,i}^{\tilde\pi^k_{h,-i}}(\tilde{s},\tilde{a}) - \mathbb{P}_{h}^{\tilde\pi_{h,-i}^{k}}\tilde{V}_{h+1,i}(\tilde{s},\tilde{a}))\big]\big\rvert  \lesssim H \sqrt{\frac{\tau}{K}\log(1/\delta)},
\end{align*}
then taking union bound over $\bar{s} \in \mathcal{N}_\epsilon^D$ and let $\epsilon = \sqrt{\lambda/K}$ leads to \begin{align}\label{eq-G2-dA}
  \lvert G_2 \rvert & \lesssim H\sqrt{\frac{\tau dA}{K}}. 
\end{align}
Now combining \eqref{eq-G2-bound-SA},\eqref{eq-G2-tau},\eqref{eq-G2-dA} leads to the desired result.

\end{proof}

Combining our bounds on $I_{h,i}^{(1)}$ and $I_{h,i}^{(2)}$, we get then \begin{align}\label{eq-recursive-I2}
  \xi_{h} \leq \xi_{h+1} + \tilde{O} \bigg( \nu \sqrt{d}  + H \sqrt{\tau}\big(\sqrt{\frac{ \log(S) }{K }} \wedge \sqrt{ d(\frac{A}{K} \wedge 1 )}\big) \bigg)
\end{align}

Applying \eqref{eq-recursive-I2} recursively with $\xi_{H+1} = 0$ leads to \begin{align}\label{eq-I2-bound}
  \xi_{1} \lesssim    H\nu \sqrt{d}+ H^2 \sqrt{\tau}\big(\sqrt{\frac{ \log(S) }{K }} \wedge \sqrt{ d(\frac{A}{K} \wedge 1 )}\big)
\end{align}

Combining \eqref{eq-I1-bound} and \eqref{eq-I2-bound} together leads to \begin{align*}
 \lvert V_{1,i}^{\dagger,\tilde{\pi}_{-i}} - V_{1,i}^{\tilde{\pi}} \rvert \lesssim H\nu \sqrt{d}+ H^2 \sqrt{\tau}\big(\sqrt{\frac{ \log(S) }{K }} \wedge \sqrt{ d(\frac{A}{K} \wedge 1 )}\big)
\end{align*}

That provides the CCE guarantee of every epoch of the virtual algorithm.

\subsection{Analysis of the Single Agent Learning Subroutine}
We would show that for each agent $i$, the single agent learning subroutine is provable to output an approximation of the best-response policy when other agents are playing according to $\tilde{\pi}_{-i}$:

\begin{lemma}\label{lem-single-agent}
    With probability at least $1-\delta$, we have Algorithm~\ref{alg:linear-single-Virtual} returns a policy $\tilde{\pi}_i^\dagger$ so that \begin{align*}
        V_{h,i}^{\dagger,\tilde{\pi}_{-i}} - V_{h,i}^{\tilde{\pi}_i^\dagger \times \tilde{\pi}_{-i}} \lesssim H^2\sqrt{\dfrac{\tau\min\{\log S,d\}}{K}} + \nu H\sqrt{d}.
    \end{align*}
\end{lemma}

\begin{proof}
As in proof of \eqref{eq-epoch-prob}, we have for every $h$,
\begin{align*}
V_{h,i}^{\dagger,\tilde{\pi}_{-i}} - V_{h,i}^{\hat{\pi}_{i}^\dagger\times \tilde{\pi}_{-i}} &= \underbrace{V_{h,i}^{\dagger,\tilde{\pi}_{-i}} - \tilde{V}_{h,i}^{\dagger}}_{:= J_{h,i}^{(1)}} + \underbrace{\tilde{V}_{h,i}^{\dagger} - V_{h,i}^{\hat{\pi}_{i}^\dagger\times \tilde{\pi}_{-i}}}_{:= J_{h,i}^{(2)}}
\end{align*}

For $J_{h,i}^{(1)},$ we have \begin{align*}
V_{h,i}^{\dagger, \tilde{\pi}_{-i}}(s) &= \max_{a}\big\{ r_{h,i}^{\tilde{\pi}_{-i}}(s,a) + \mathbb{P}_{h,i}^{\pi_{-i}}V_{h+1,i}^{\dagger, \tilde{\pi}_{-i}}(s,a)\big\}\\
  &\leq \max_{a}\big\{ r_{h,i}^{\tilde{\pi}_{-i}}(s,a) +  \mathbb{P}_{h,i}^{\pi_{-i}}\tilde{V}_{h+1,i}^\dagger(s,a)\big\} + \max_{s}\big(V_{h+1,i}^{\dagger,\tilde{\pi}_{-i}}(s) - \tilde{V}_{h+1,i}^\dagger(s,a) \big)\\
  &\leq \tilde{V}_{h,i}^\dagger(s) +  \big(\max_{a}\big\{ r_{h,i}^{\tilde{\pi}_{-i}}(s,a) +  \mathbb{P}_{h,i}^{\pi_{-i}}\tilde{V}_{h+1,i}^\dagger(s)\big\} - \tilde{V}_{h,i}^\dagger(s) \big) + \max_{s}\big(V_{h+1,i}^{\dagger,\tilde{\pi}_{-i}}(s) - \tilde{V}_{h+1,i}^\dagger(s) \big)\\
  &\leq \tilde{V}_{h,i}^\dagger(s) +  \max_{a}\big\{ r_{h,i}^{\tilde{\pi}_{-i}}(s,a) +  \mathbb{P}_{h,i}^{\pi_{-i}}\tilde{V}_{h+1,i}^\dagger(s,a) - \tilde{Q}_{h,i}^\dagger(s,a)  \big\} + \max_{s}\big(V_{h+1,i}^{\dagger,\tilde{\pi}_{-i}}(s,a) - \tilde{V}_{h+1,i}^\dagger(s) \big)
  \end{align*}
Now we can bound $$\max_{s,a}\big\{ r_{h,i}^{\tilde{\pi}_{-i}}(s,a) +  \mathbb{P}_{h,i}^{\pi_{-i}}\tilde{V}_{h+1,i}^\dagger(s,a) - \tilde{Q}_{h,i}^\dagger(s,a)  \big\}$$ using the property of coreset and the martingale concentration argument as in bounding $J_1,J_2,J_3$ in section~\ref{sec-proof-lem2}, we have then with probability at least $1-\delta$, \begin{align*}
   \max_s J_{h,i}^{(1)}(s) \leq \max_s J_{h+1,i}^{(1)}(s) + \tilde{O}(H\sqrt{\frac{\tau \min\{d,\log S\}}{K}} + \nu\sqrt{d})\lesssim H^2\sqrt{\frac{\tau \min\{d,\log S\}}{K}} + \nu H\sqrt{d},\quad \forall h \in [H].
\end{align*}

\noindent For $J^{(2)}_{h,i},$ noticing that \begin{align*}
  \lvert  \tilde{V}^{\dagger}_{h,i}(s) - {V}_{h,i}^{\tilde{\pi}_i^\dagger \times \tilde{\pi}_{-i}}(s) \rvert  &= \tilde{V}^{\dagger}_{h,i}(s) - \big( r_{h,i}^{\tilde{\pi}_{-i}}(s,\tilde{\pi}^\dagger_i(s)) + \mathbb{P}_{h,i}^{\tilde{\pi}_{-i}}{V}_{h+1,i}^{{\tilde{\pi}_i^\dagger \times \tilde{\pi}_{-i}}}(s,\tilde{\pi}^\dagger_i(s))\big)\\
   &\leq \big\lvert \tilde{Q}^{\dagger}_{h,i}(s,\tilde{\pi}_i^\dagger(s)) - \big( r_{h,i}^{\tilde{\pi}_{-i}}(s,\tilde{\pi}^\dagger_i(s)) + \mathbb{P}_{h,i}^{\tilde{\pi}_{-i}}\tilde{V}_{h+1,i}^{\dagger}(s,\tilde{\pi}^\dagger_i(s)) \big) \big\rvert  + \max_{s}  \lvert  \tilde{V}^{\dagger}_{h+1,i}(s) - {V}_{h+1,i}^{\tilde{\pi}_i^\dagger \times \tilde{\pi}_{-i}}(s) \rvert\\
   &\leq \max_{s,a}\big\lvert \tilde{Q}^{\dagger}_{h,i}(s,a) - \big( r_{h,i}^{\tilde{\pi}_{-i}}(s,a) + \mathbb{P}_{h,i}^{\tilde{\pi}_{-i}}\tilde{V}_{h+1,i}^{\dagger}(s,a) \big) \big\rvert  + \max_{s}  \lvert  \tilde{V}^{\dagger}_{h+1,i}(s) - {V}_{h+1,i}^{\tilde{\pi}_i^\dagger \times \tilde{\pi}_{-i}}(s) \rvert.\\
   \end{align*}
   Then similar to $J_{h,i}^{(1)},$ we get 
   \begin{align*}
   \max_s J_{h,i}^{(2)}(s) \leq \max_s J_{h+1,i}^{(2)}(s) + \tilde{O}(H\sqrt{\frac{\tau \min\{d,\log S\}}{K}} + \nu\sqrt{d})\lesssim H^2\sqrt{\frac{\tau \min\{d,\log S\}}{K}} + \nu H\sqrt{d},\quad \forall h \in [H].
\end{align*}
\end{proof}

\subsection{Analysis of the Main Algorithm}

Now we would derive the $\varepsilon$-CCE guarantee of the main algorithm by bridging its performance to the virtual algorithm.

Firstly, at the last round of the main algorithm(means it succesfully returns the policy without restarting), the outputed policy $\hat\pi$ \textbf{at every $s\in \mathcal{C}_{h}^\ell$} has the same performance as the $\tau$-th virtual algorithm for some $1\leq \tau \leq L$, for which we denote by $\tilde{\pi}.$
 In particular, since $(s_0,a) \in \mathcal{C}_1^\tau$ for all $a\in A$ by definition, we have then by the union bound arguement, with high probability\begin{align*}
   V^{\dagger,\tilde{\pi}_{-i}}_{1,i}(s_0) - V^{\tilde{\pi}}_{1,i}(s_0) \leq \varepsilon + c \nu H\sqrt{d}, \quad \forall i\in [m].
\end{align*}
Now we still need to bridge   $ V^{\dagger,\tilde{\pi}_{-i}}_{1,i}(s_0)$ to $ V^{\dagger,\hat{\pi}_{-i}}_{1,i}(s_0)$ and  $ V^{\dagger,\tilde{\pi}}_{1,i}(s_0)$ to $ V^{\hat{\pi}}_{1,i}(s_0).$ To do that, we define another virtual algorithm(called quasi algorithm) that nearly same as the algorithm in section~\ref{sec-virtual-alg}, except that for all $s\in S,$ we learn all $\bar{\pi}$  from previously defined $\hat{V}_{h+1,i}$ for every $h,i.$
Since this algorithm also coupled with the main algorithm, we have its output policy at $\tau$-th epoch is same as the main algorithm \textbf{at every $s\in S$}. Thus for every $h,i,s,a$ we have \begin{align*}
   Q^{\hat\pi}_{h,i}(s,a)  = Q^{\bar\pi}_{h,i}(s,a).
\end{align*}

Thus we need only provide the guarantee of the quasi algorithm: 

At every epoch of the Quasi algorithm, for each agent $i$ and its core-set $\mathcal{D}_{h,i}$, consider the $N$ reward paths $\{\bar{r}_{h,i,n}(\tilde{s},\tilde{a}) \}$ generated in the second last uncertainty check loop, then we have by i.i.d. concentration it holds w.h.p.
\begin{align*}
 \lvert  V_{1,i}^{\bar{\pi}}(s_0) - \frac{1}{N}\sum_{n=1}^N \bar{r}_{h,i,n}(\tilde{s},\tilde{a}) \rvert    \lesssim \frac{H}{\sqrt{N}}.
\end{align*}

Taking union bounds over epochs, we have the result holds for every epoch. By the same argument, such result also holds for the quasi algorithm.

Now if we consider the $\tau$-th epoch, we have then \begin{align}\label{eq-Q-error}
   \big \lvert  V^{\hat{\pi}}_{1,i}(s_0) - V^{\tilde{\pi}}_{1,i}(s_0) \big \rvert \lesssim  \frac{H}{\sqrt{N}} .
\end{align}

Combining  \eqref{eq-Q-error} with the $\varepsilon$-CCE guarantee of $\tilde{\pi}$, we have
\begin{align*}
  \lvert V_{1,i}^{\hat\pi}(s_0) - V_{1,i}^{\dagger,\tilde{\pi}_{-i}}(s_0)\rvert \leq \varepsilon + c\big( \nu H\sqrt{d} + \frac{H}{\sqrt{N}}\big)
\end{align*}
Now it is sufficient to bridge $ V_{1,i}^{\dagger,\tilde{\pi}_{-i}}(s_0)$ with $ V_{1,i}^{\dagger,\hat{\pi}_{-i}}(s_0)$, to do that, we consider another virtual algorithm, called the virtual-II algorithm, its operation on learning $\hat{\pi}$ is same as the quasi algorithm, the main difference is its single-agent learning procedure: for each $i,h$, it maintains a complementary core dataset $\tilde{D}_{h,i}$ in the same way as the virtual I algorithm, and taking LSVI using the collected complementary data for $s\notin \mathcal{C}_h$  while do the same LSVI as the main algorithm for $s\in \mathcal{C}_h.$ Applying Lemma~\ref{lem-single-agent} leads to the following error guarantee for every epoch output policy of the virtual-I, virtual-II algorithm:\begin{align}\label{eq-dagger-hat-error}
  \lvert V^{\tilde{\pi}^{\dagger}_i \times \tilde{\pi}_{-i}}_{1,i}(s_0) - V^{\dagger,\tilde{\pi}_{-i}}_{1,i}(s_0) \rvert \leq \varepsilon + c\nu H\sqrt{d},\quad  \lvert V^{\hat{\pi}^{\dagger}\times \hat{\pi}_{-i}}_{1,i}(s_0) - V^{\dagger,\hat{\pi}_{-i}}_{1,i}(s_0) \rvert \leq \varepsilon + c\nu H\sqrt{d}.
\end{align}
On the other hand, the last rollout procedure for every $i$ guarantees at the $\tau$-th epoch, \begin{align}\label{eq-dagger-hat-concentration}
  \lvert V^{\hat{\pi}^{\dagger,i}}_{1,i}(s_0) - V^{\tilde{\pi}^{\dagger,i}}_{1,i}(s_0) \rvert \leq \frac{H}{\sqrt{N}}  
\end{align}

Combining \eqref{eq-dagger-hat-error} and \eqref{eq-dagger-hat-concentration}, we have $
  \lvert V_{1,i}^{\dagger, \tilde{\pi}_{-i}}(s_0) - V_{1,i}^{\dagger,\hat{\pi}_{-i}}(s_0) \rvert \lesssim \varepsilon + \frac{H}{\sqrt{N}}.$ That then leads to the desired bound $
  \lvert V_{1,i}^{\hat\pi}(s_0) - V_{1,i}^{\dagger,\hat{\pi}_{-i}}(s_0)\rvert \lesssim \varepsilon + \frac{H}{\sqrt{N}}+ c\nu H\sqrt{d}.
$ Finally letting $N \asymp H^2/\varepsilon^2$ leads to the desired result.

%% file: appendix_generative.tex
\section{Results under the Random Access Model}\label{appendix-generative}
\subsection{Algorithm under the Random Access Model}

We propose the algorithm under the random access model in  Algorithm~\ref{alg:linear-generative} and make several remarks.

\begin{remark}
    When letting $\beta_{h,i} = 0$ and $\alpha_k = \frac{1}{k}$, the update formulas of $\hat{Q}_{h,i}$ and $\hat{V}_{h,i}$ is same as those in Algorithm~\ref{alg:UC-linear} and Algorithm~\ref{alg:UC-linear-Virtual}.
\end{remark}

\begin{remark}
Compared with the algorithm under the local access model, Algorithm~\ref{alg:linear-generative} doesn't contain the Policy-Rollout subroutine(Line~15 to Line~18 and Line~27 to Line~32) in Algorithm~\ref{alg:Main}. The main reason is that the random access protocol makes the algorithm easy have high confidence to all states after the exploration phase in line~2. 
\end{remark}
\begin{remark}
  When we consider the tabular case, i.e. $\phi_{i}(s,a) = \bm{e}_{s,a} \in \mathbb{R}^{S A_i}$, Algorithm~\ref{alg:linear-generative} with $$\lambda = 0,\quad \tau = 1, \quad \alpha_k = \frac{c_\alpha\log K}{k-1+c_\alpha\log K}, \quad \beta_{i,h} = c_b \sqrt{\frac{\log^3(\frac{KS\sum_i A_i}{\delta}}{KH}}\sum_{k=1}^K \alpha_k^K \bigg\{ \text{Var}_{\pi^k_{i,h}(\cdot \lvert s}\big( q^k_{i,h}(s,\cdot) \big)+H \bigg\}  $$ with \begin{align*}
      \alpha^k_i = \alpha_i\prod_{j=i+1}^k (1-\alpha_j) \text{ if }0<i<k,\quad \alpha_i^k = \alpha_k \text{ if }i=k
  \end{align*} recovers the algorithm proposed in \cite{li2022minimax}. They have shown that such selection of parameter allows the algorithm to learn a $\varepsilon$-CCE with $\tilde{O}(\frac{H^4S \sum_i A_i}{\varepsilon^2})$ sample complexity.
\end{remark}

\newpage
\begin{algorithm}[H]
  \label{alg:linear-generative}
  \SetKwInOut{Input}{Input}
  \SetKwInOut{Output}{output}
  \caption{Linear Game Random Access}
  \SetAlgoLined
  \Input{learning rates $\{\alpha_k\}$ and $\{\eta_{k+1}\}$}
  \For{$i=1$ \KwTo $m$}{
  select $ \mathcal{D}_i \subset S\times A_i,i\in [m] $ such that 
  $$\max_{s,a}\phi_i(s,a)\big(\sum_{\bar{s}\in \mathcal{C}_i}\sum_{\bar{a} \in \mathcal{A}_i} \phi_i(\bar s,\bar a)\phi_i(\bar s,\bar a)^\top + \lambda I \big)^{-1} \phi_i(s,a)   < \tau,\quad \forall i\in[m]. $$
  }
  \For{$h = H$ \KwTo $1$}{
   \For{$k = 1$ \KwTo $K$}{
    \For{$i = 1$ \KwTo $m$}{
      \For{$(s,a)\in \mathcal{D}_i$}{
      $(r,s') \leftarrow \text{local sampling}(i, s, a, \pi^k_{h,-i} )$    
     Compute $q_{h,i}^k(s,a) = r + \hat{V}_{h+1,i}(s')$.  
    }
    \begin{align*}
      &\theta^k = \text{argmin}_{\theta} \sum_{(s,a)\in \mathcal{D}_i} \lvert q^{k}_{h,i}(s,a) -  \langle\phi_{i}(s,a), \theta \rangle \rvert_2^2 + \lambda\lVert \theta \rVert_2^2 \\
      &{Q}_{h,i}^k(s,a) =  \langle \phi_i(s,a),(1-\alpha_k)\theta^{k-1}+\alpha_k\theta^k\rangle.
       \end{align*}
      $$ \pi_{h,i}^{k+1}(a_i | s) = \frac{\exp(\eta_{k+1} Q_{h,i}^k(s,a_i))}{\sum_{a'}\exp(\eta_{k+1} Q_{h,i}^{k}(s,a'))}, \quad \forall s \in S.$$
    }
     }
    \For{$i = 1$ \KwTo $m$}{
     $$\hat{V}_{h,i}(s) = \min\left\{ \sum_{k=1}^{K} \alpha_k^K \langle \pi_{h,i}^{k'}, q_{h,i}^{k'}(s,\cdot) \rangle + \beta_{h,i}(s), H-h+1 \right\}, \forall s \in S.$$ 
    }
    }
    \Return{$\hat{\pi}_{h,i}:= \sum_{k=1}^K \alpha_k^K \pi_{h,i}^k$.}
\end{algorithm}

\subsection{Proof of Theorem~\ref{thm-generative}}

Firstly, we would note that when $\beta = 0$, Algorithm~\ref{alg:linear-generative} is nearly same as the Algorithm~\ref{alg:UC-linear-Virtual} despite a slight difference on the construction of the coreset $\mathcal{D}_{h,i}$. And during the proof of Lemma~\ref{lem-I1-bound} and Lemma~\ref{lem-I2-bound}, the only property we have required for the $\mathcal{D}_{h,i}$ can be summarized as the following:\begin{align*}
    \lvert \mathcal{D}_{h,i}\rvert \lesssim \frac{d(1+\tau)}{\tau}, \quad  \sup_{s,a}\phi_i(s,a)(\sum_{\bar{s},\bar{a} \in \mathcal{\mathcal{D}_{h,i}}} \phi(\bar{s},\bar{a})\phi(\bar{s},\bar{a})^\top +\lambda I)^{-1} \phi_i(s,a) < \tau.
\end{align*}
And such property is straightforward to verify for $\mathcal{D}_{h,i} \equiv \mathcal{D}_i.$ Thus the analysis in Appendix~C.2 and the result in \eqref{eq-epoch-prob} can be applied for Algorithm~\ref{alg:linear-generative}. To prove Theorem~\ref{thm-generative}, it suffice to specify the selection of parameters based on \eqref{eq-epoch-prob}: 

1. When $\min\{ d^{-1}\log S,A \} \leq \varepsilon^{-2}$: letting \begin{align*}
K = \tilde{O}(H^4 d \varepsilon^{-2}\min\{d^{-1}\log S,A\}), \tau = 1
\end{align*}
leads to the $\varepsilon + c\nu\sqrt{d}H $-CCE guarantee, in this case, the total sample complexity is given by $\tilde{O}(KmdH) = \tilde{O}(\varepsilon^{-2}H^5 d^2 m \min\{d^{-1}\log S, A\})$.

\noindent 2. When $\min\{ d^{-1}\log S,A \} > \varepsilon^{-2}:$ letting \begin{align*}
K = \tilde{O}(H^4d\varepsilon^{-2}), \tau = \tilde{O}(H^{-4}\varepsilon^2 d^{-1})
\end{align*}
leads to the $\varepsilon + cH\nu\sqrt{d}$-CCE guarantee, in this case, the total sample complexity is given by $\tilde{O}(KmdH/\tau) = \tilde{O}(\varepsilon^{-4} H^9d^3). $

Combining the sample complexity of this two cases leads to the $\tilde{O}(\min\{\varepsilon^{-2}dH^4,d^{-1}\log S, A\} d^2 H^5 m\varepsilon^{-2})$ sample complexity result.

%% file: appendix_aux_lem.tex
\section{Proof of Auxiliary Results}
\subsection{Proof of Lemma~\ref{lem-ftrl-our-bound}}

\begin{proof}
We recall the following standard regret result of FTRL \cite{Lattimore2020BanditA}: 
\begin{lemma}\label{lem-FTRL}
  For a sequence $\{y_t\}_{t=1}^T \in [0,1]^A$ and the policy sequence generated by $$\pi_{t+1,a}\propto \exp(-\eta \sum_{k=1}^t y_{ka})$$
  with $\eta = \sqrt{2\log(A)/T}$, it holds that\begin{align*}
    \max_{a} \frac{1}{T}\sum_{t=1}^T\big(  \langle \pi_t, y_t\rangle -  y_{ta} \big)\leq \sqrt{2\frac{\log A}{T}}
  \end{align*}
\end{lemma}

Now since it holds the following lemma regarding the bound of $Q_{h,i}^k$: \begin{lemma}\label{lem-Q-bound}
  With probability at least $1-\delta,$ we have \begin{align*}
    \max_{s,a} \lvert \tilde{Q}_{h,i}^k(s,a) \rvert \lesssim H\min \{\sqrt{d}, 1+\sqrt{\log(SA/\delta)} + \nu \sqrt{d}\}.
  \end{align*}
\end{lemma} 
\begin{proof}[Proof of Lemma~\ref{lem-Q-bound}] 
  Denote \begin{align*}
    \theta_{h,i}^{\pi_{h,-i}^k,\tilde{V}_{h+1,i}}:= \text{argmin}_{\lVert \theta\rVert_2\leq H\sqrt{d}} \lVert \phi_i(s,a)^\top\theta - Q_{h,i}^{\tilde{\pi}_{-i,h}^k,\tilde{V}_{h+1,i}} \rVert_\infty
  \end{align*}
  \textbf{1. When }$s \in \mathcal{C}_h$: 
  \begin{align*}
  \tilde{Q}_{h,i}^k(s,a) &= \phi_i(s,a)^\top \Lambda_{h,i}^{-1}\sum_{(\tilde{s},\tilde{a})\in \mathcal{D}_{h,i}} \phi_i(\tilde{s},\tilde{a}) q_{h,i}^k(\tilde{s},\tilde{a})\\
  &= \phi_i(s,a)^\top \Lambda_{h,i}^{-1}\sum_{(\tilde{s},\tilde{a})\in \mathcal{D}_{h,i}} \phi_i(\tilde{s},\tilde{a}) [q_{h,i}^k(\tilde{s},\tilde{a})\pm \phi_i(\tilde{s},\tilde{a})^\top \theta_{h,i}^{\pi_{h,-i}^k,\tilde{V}_{h+1,i}}]\\
  &= Q_{h,i}^{\pi^k_{-i,h},\tilde{V}_{h+1,i}}(s,a)+\phi_i(s,a)^\top \Lambda_{h,i}^{-1}\sum_{(\tilde{s},\tilde{a})\in \mathcal{D}_{h,i}} \phi_i(\tilde{s},\tilde{a}) [q_{h,i}^k(\tilde{s},\tilde{a})- \phi_i(\tilde{s},\tilde{a})^\top \theta_{h,i}^{\pi_{h,-i}^k,\tilde{V}_{h+1,i}}]  + O(\nu+H\sqrt{\tau \lambda d})\\
  &=Q_{h,i}^{\pi^k_{-i,h},\tilde{V}_{h+1,i}}(s,a) + O( \nu\sqrt{d\log d} + H\sqrt{\tau \log(1/\delta)} )\\
  &= O\big(H(1+\sqrt{\tau \log(1/\delta)})+ \nu \sqrt{d\log d}\big),
\end{align*}
Where the last second line is by \begin{align*}
    & \phi_i(s,a)^\top \Lambda_{h,i}^{-1}\sum_{(\tilde{s},\tilde{a})\in \mathcal{D}_{h,i}} \phi_i(\tilde{s},\tilde{a}) [q_{h,i}^k(\tilde{s},\tilde{a})- \phi_i(\tilde{s},\tilde{a})^\top \theta_{h,i}^{\pi_{h,-i}^k,\tilde{V}_{h+1,i}}]\\
    =& \phi_i(s,a)^\top \Lambda_{h,i}^{-1} \sum_{(\tilde{s},\tilde{a})\in \mathcal{D}_{h,i}} \phi_i(\tilde{s},\tilde{a}) \big[\underbrace{Q_{h,i}^{\pi^k_{-i,h}}(\tilde{s},\tilde{a}) - \phi_i(\tilde{s},\tilde{a})^\top \theta_{h,i}^{\pi_{h,-i}^k,\tilde{V}_{h+1,i}}}_{\nu^k_{h,i}(\tilde{s},\tilde{a})} + \underbrace{q_{h,i}^k(\tilde{s},\tilde{a}) - Q_{h,i}^{\pi^{k}_{-i, h}}(\tilde{s},\tilde{a})}_{\mu^k_{h,i}(\tilde{s},\tilde{a})} \big]
\end{align*}
and \begin{align*}
   \lvert  \phi_i(s,a)^\top \Lambda_{h,i}^{-1} \sum_{(\tilde{s},\tilde{a})\in \mathcal{D}_{h,i}} \phi_i(\tilde{s},\tilde{a})\nu_{h,i}^k(\tilde{s},\tilde{a}) \rvert  &\leq  \sqrt{\sum_{\tilde{s},\tilde{a}\in \mathcal{D}_{h,i}}[\phi_i(s,a)^\top \Lambda_{h,i}^{-1} \phi_i(\tilde{s},\tilde{a})]^2}  \sqrt{\lvert \mathcal{D}_{h,i}\rvert  } \nu\\
   &\lesssim \sqrt{\tau}\cdot \sqrt{\frac{d\log d}{\tau}}\nu.
\end{align*}
and with probability at least $1-\delta$, \begin{align*}
       \lvert  \phi_i(s,a)^\top \Lambda_{h,i}^{-1} \sum_{(\tilde{s},\tilde{a})\in \mathcal{D}_{h,i}} \phi_i(\tilde{s},\tilde{a})\mu_{h,i}^k(\tilde{s},\tilde{a}) \rvert  &\leq  H\sqrt{\sum_{\tilde{s},\tilde{a}\in \mathcal{D}_{h,i}}[\phi_i(s,a)^\top \Lambda_{h,i}^{-1} \phi_i(\tilde{s},\tilde{a})]^2 \log(1/\delta)} \\
   &\lesssim H\sqrt{\tau\log(1/\delta)}.
\end{align*}
That for any fixed $(s,a)\in \mathcal{C}_h\times \mathcal{A}_i$, with probability at least $1-\delta$,\begin{align*}
  \lvert  \tilde{Q}_{h,i}^k(s,a)\rvert  \leq c  \big[ H(1+\sqrt{\tau \log(SA/\delta)})+ \nu \sqrt{d\log d}\big]
\end{align*}
On the other hand, we have the following determinstic bound:\begin{align*}
  \tilde{Q}_{h,i}^k(s,a) &= \phi_i(s,a)^\top \Lambda_{h,i}^{-1}\sum_{(\tilde{s},\tilde{a})\in \mathcal{D}_{h,i}} \phi_i(\tilde{s},\tilde{a}) q_{h,i}^k(\tilde{s},\tilde{a})\\
  &\leq H \sqrt{\lvert \mathcal{D}_{h,i}\rvert } \cdot \sqrt{\sum_{\tilde{s},\tilde{a}\in \mathcal{D}_{h,i}}[\phi_i(s,a)^\top \Lambda_{h,i}^{-1} \phi_i(\tilde{s},\tilde{a})]^2 }\\
  &\lesssim H \sqrt{d}.
\end{align*}
Thus for any fixed $(s,a)\in \mathcal{C}_h\times \mathcal{A}_i$, with probability at least $1-\delta$, \begin{align*}
  \lvert  \tilde{Q}_{h,i}^k(s,a)\rvert  \lesssim  H \min\{1+\sqrt{\tau \log(1/\delta)})+ \nu \sqrt{d}, \sqrt{d}\}
\end{align*}

\textbf{2. When }$s \notin \mathcal{C}_h$:

By our construction of $\tilde{\mathcal{D}}_{h,i},\tilde{\Lambda}_h^{-1}$ in the virtual algorithm, it holds that \begin{align*}
   \sqrt{\lvert \tilde{D}_{h,i}\rvert} \lesssim \sqrt{d/\tau} ,\quad \lVert \phi_i(s,a)\rVert_{\tilde{\Lambda}_{h,i}^{-1}} \leq \tau, \forall s\notin \mathcal{C}_{h}
\end{align*}  
  thus our argument when $s\in \mathcal{C}_h$ still holds by replacing $\mathcal{D}_{h,i},{\Lambda}_h^{-1}$ by $\tilde{\mathcal{D}}_{h,i},\tilde{\Lambda}_h^{-1}$. 
  
  i.e. for any fixed $(s,a)\notin \mathcal{C}_h\times \mathcal{A}_i$, with probability at least $1-\delta$, \begin{align*}
    \lvert  Q_{h,i}^k(s,a)\rvert  \lesssim  H \min\{1+\sqrt{\tau \log(1/\delta)})+ \nu \sqrt{d}, \sqrt{d}\}
  \end{align*}
    
Now, taking union bound on $(s,a)$, we have with probability at least $1-\delta,$
\begin{align*}
    \lvert  Q_{h,i}^k(s,a)\rvert  \lesssim  H \min\{1+\sqrt{\tau \log(SA/\delta)}+ \nu \sqrt{d}, \sqrt{d}\}.
\end{align*}

\end{proof}
Denote $\gamma: = \min\{1+\sqrt{\tau \log(SA/\delta)}+ \nu \sqrt{d}, \sqrt{d}\}$, by Lemma~\ref{lem-Q-bound} we have there exists some absolute number $c>0$ so that with probability at least $1-\delta$,\begin{align*}
  \tilde{y}_k(s,a):= \frac{cH \gamma - Q_{h,i}^k(s,a)}{2cH\gamma} \in [0,1] \quad \forall (s,a)\in \mathcal{S} \times \mathcal{A}_i.
\end{align*} 

Now by Lemma~\ref{lem-FTRL}, for the policy sequence generated by \begin{equation}\label{eq-virtual-policy-update}
  \pi_{k,a}(a\lvert s)\propto \exp(-\eta \sum_{t=1}^k \tilde{y}_{t}(s,a))\propto \exp(\frac{\eta}{2cH\gamma} \sum_{t=1}^k \tilde{Q}_{h,i}^k(s,a))\propto  \exp(\eta_k \tilde{\bar{Q}}_{h,i}^k(s,a))\end{equation}
  with $\eta = \sqrt{2\log(A_i)/K}$, it holds that\begin{align}\label{eq-ftrl-2}
    \max_{a} \frac{1}{K}\sum_{k=1}^K\big(  \langle \pi_{k}(\cdot \lvert s), \tilde{y}_k(s,a)\rangle -  \tilde{y}_k(s,a) \big)\leq \sqrt{2\frac{\log A_i}{K}}
  \end{align}

  Multiplying $2c\gamma$ to both sides of  \eqref{eq-ftrl-2} and noticing that the iteration formula in \eqref{eq-ftrl-2} is exactly the formula for updating $\tilde{\pi}_{h,i}^{k}$ in Algorithm~\ref{alg:UC-linear-Virtual}, we get with probability at least $1-\delta$,
\begin{align*}
  \frac{1}{K}\sum_{k=1}^K \E_{a\sim \pi_{h,i}^k}[\tilde{Q}_{h,i}^k(s,a)] \leq \frac{1}{K}\min_{a} \sum_{k=1}^K \tilde{Q}_{h,i}^k(s,a) + 2c\gamma H \sqrt{\frac{2\log A_i}{K}},
\end{align*}

That leads to the desired result.
\end{proof}

\subsection{Proof of Lemma~\ref{lem-J1-bound}}
\begin{proof}[Proof of Lemma] 
For any $s,a$ we have denote $\mathcal{F}_k(\tilde{s},\tilde{a})$ the filtration genearted by the information before taking the $k$-th time sampling on $\tilde{s},\tilde{a}$, then for \begin{align*}
  Z_{k}(\tilde{s},\tilde{a}) := \phi_i(s,a)^\top \Lambda_{h,i}^{-1}\phi_i(\tilde{s},\tilde{a}) \bar{q}_{h,i}^k(\tilde{s},\tilde{a}),
\end{align*}
it holds that $ \E[Z_{k}(\tilde{s},\tilde{a})\lvert \mathcal{F}_{k}(\tilde{s},\tilde{a})] =  \phi_i(s,a)^\top \Lambda_{h,i}^{-1}\phi_i(\tilde{s},\tilde{a})\big( r_{h,i}^{\pi^k_{h,-i}}(\tilde{s},\tilde{a}) + \mathbb{P}_{h}^{\pi_{h,-i}^{k}}\widehat{V}_{h+1,i}(\tilde{s},\tilde{a}) \big)$, and $\lvert Z_k(\tilde{s},\tilde{a}) \rvert \leq H \lvert \phi_i(s,a)^\top \Lambda_{h,i}^{-1}\phi_i(\tilde{s},\tilde{a}) \rvert $ a.s., thus applying Azuma-Hoeffding's inequality leads to with probability at least $1-\delta$,
\begin{align*}
  &\frac{1}{K}\sum_{(\tilde{s},\tilde{a}) \in \mathcal{D}_{h,i} }\phi_i(s,a)^\top \Lambda_{h,i}^{-1} \phi_i(\tilde{s},\tilde{a})\big[\sum_{k=1}^K (\bar{q}_{h,i}^k(\tilde{s},\tilde{a}) - r_{h,i}^{\pi^k_{h,-i}}(\tilde{s},\tilde{a}) - \mathbb{P}_{h}^{\pi_{h,-i}^{k}}\widehat{V}_{h+1,i}(\tilde{s},\tilde{a}))\big]\\
=&\frac{1}{K}\sum_{(\tilde{s},\tilde{a}) \in \mathcal{D}_{h,i} }\sum_{k=1}^K \big( Z_{k}(\tilde{s},\tilde{a}) - \E[Z_k(\tilde{s},\tilde{a})\lvert \mathcal{F}_{k}(\tilde{s},\tilde{a})] \big)\\
\lesssim & \frac{H}{K}\sqrt{\sum_{(\tilde{s},\tilde{a}) \in \mathcal{D}_{h,i} }\sum_{k=1}^K  \lvert \phi_i(s,a)^\top \Lambda_{h,i}^{-1}\phi_i(\tilde{s},\tilde{a}) \rvert^2 \log(1/\delta)}.
\end{align*} 

On the other hand, we have \begin{align*}
  \sum_{(\tilde{s},\tilde{a}) \in \mathcal{D}_{h,i} }\sum_{k=1}^K  \lvert \phi_i(s,a)^\top \Lambda_{h,i}^{-1}\phi_i(\tilde{s},\tilde{a}) \rvert^2
  =&\sum_{(\tilde{s},\tilde{a}) \in \mathcal{D}_{h,i} }\sum_{k=1}^K   \phi_i(s,a)^\top \Lambda_{h,i}^{-1}\phi_i(\tilde{s},\tilde{a})\phi_i(\tilde{s},\tilde{a}))^\top \Lambda_{h,i}^{-1}  \phi_i({s},{a})\\
  \leq& \sum_{k=1}^K   \phi_i(s,a)^\top \Lambda_{h,i}^{-1}\phi_i({s},{a}) \\
  \leq&  K \tau.
\end{align*}
Taking union bound over $\mathcal{C}_h \times \mathcal{A}_i$ leads to with probability at least $1-\delta$, \begin{align}\label{eq-J1-bound-SA}
  J_1 \lesssim H \sqrt{\frac{\tau \log(SA/\delta)}{K}}, \quad \forall (s,a) \in \mathcal{S}\times \mathcal{A}_i.
\end{align}

On the other hand, we have it holds for all $(s,a)$ that  
\begin{align*}
J_1 &=\frac{1}{K}\sum_{(\tilde{s},\tilde{a}) \in \mathcal{D}_{h,i} }\phi_i(s,a)^\top \Lambda_{h,i}^{-1} \phi_i(\tilde{s},\tilde{a})\big[\sum_{k=1}^K ({q}_{h,i}^k(\tilde{s},\tilde{a}) - Q_{h,i}^{\pi_{h,-i}^k,\tilde{V}_{h+1,i}}(\tilde{s},\tilde{a}) )\big]\\
    &\leq \frac{1}{K} \lVert \phi_i(s,a)\rVert_{\Lambda_{h,i}^{-1}}  \big\lVert \sum_{(\tilde{s},\tilde{a}) \in \mathcal{D}_{h,i}}\sum_{k=1}^K \phi_i(\tilde{s},\tilde{a})\big({q}_{h,i}^k(\tilde{s},\tilde{a}) - Q_{h,i}^{\pi_{h,-i}^k,\tilde{V}_{h+1,i}}(\tilde{s},\tilde{a}) \big)\big\rVert_{\Lambda_{h,i}^{-1}}\\
    &\leq \frac{\sqrt{\tau}}{K}\big\lVert\sum_{(\tilde{s},\tilde{a}) \in \mathcal{D}_{h,i}}\sum_{k=1}^K \phi_i(\tilde{s},\tilde{a})\big({q}_{h,i}^k(\tilde{s},\tilde{a}) - Q_{h,i}^{\pi_{h,-i}^k,\tilde{V}_{h+1,i}}(\tilde{s},\tilde{a}) \big)\big\rVert_{\Lambda_{h,i}^{-1}}.
\end{align*}

Now noticing that \begin{align*}
  &\big\lVert\sum_{(\tilde{s},\tilde{a}) \in \mathcal{D}_{h,i}}\sum_{k=1}^K \phi_i(\tilde{s},\tilde{a})\big({q}_{h,i}^k(\tilde{s},\tilde{a}) - Q_{h,i}^{\pi_{h,-i}^k,\tilde{V}_{h+1,i}}(\tilde{s},\tilde{a}) \big)\big\rVert_{\Lambda_{h,i}^{-1}}\\
  =&\big\lVert \sum_{k=1}^K \sum_{(\tilde{s},\tilde{a}) \in \mathcal{D}_{h,i}} \Lambda_{h,i}^{-1/2} \phi_i(\tilde{s},\tilde{a})\big({q}_{h,i}^k(\tilde{s},\tilde{a}) - Q_{h,i}^{\pi_{h,-i}^k,\tilde{V}_{h+1,i}}(\tilde{s},\tilde{a}) \big)\big\rVert_{2}\\
  =& \sup_{\lVert v \rVert_2 = 1} \sum_{k=1}^K \sum_{(\tilde{s},\tilde{a}) \in \mathcal{D}_{h,i}} v^\top\Lambda_{h,i}^{-1/2} \phi_i(\tilde{s},\tilde{a})\big({q}_{h,i}^k(\tilde{s},\tilde{a}) - Q_{h,i}^{\pi_{h,-i}^k,\tilde{V}_{h+1,i}}(\tilde{s},\tilde{a}) \big).
\end{align*}
Denote $\mathbb{S}_{\mathcal{H}} = \{g\in \mathbb{R}^d: \lVert g \rVert_2 = 1 \}$ , then for any fixed $g\in \mathbb{S}_{\mathcal{H}}$, we have by Azuma-Hoeffding inequality, with probability at least $1-\delta,$  \begin{align*}
  &\sum_{k=1}^K \sum_{(\tilde{s},\tilde{a}) \in \mathcal{D}_{h,i}} g^\top\Lambda_{h,i}^{-1/2} \phi_i(\tilde{s},\tilde{a})\big({q}_{h,i}^k(\tilde{s},\tilde{a}) - Q_{h,i}^{\pi_{h,-i}^k,\tilde{V}_{h+1,i}}(\tilde{s},\tilde{a}) \big)\\
  \lesssim & H\sqrt{\sum_{k=1}^K \sum_{(\tilde{s},\tilde{a})\in \mathcal{D}_{h,i}} (g^\top\Lambda_{h,i}^{-1/2}\phi_i(\tilde{s},\tilde{a}) )^2 \log(1/\delta)  }\leq H\sqrt{K \log(1/\delta) }.
\end{align*}

If we consider the minimal $\epsilon$-net $\mathcal{N}_\epsilon$ of $\mathbb{S}_{\mathcal{H}}$,  i.e. \begin{align*}
  \forall g \in \mathbb{S}_{\mathcal{H}}, \exists g_0 \in \mathcal{N}_\epsilon \text{ such that } \lVert g - g_0 \rVert_2 \leq \epsilon.
\end{align*}
In particular for any $g,g'\in \mathbb{S}_{\mathcal{H}}, $ we have \begin{align*}
  &\sum_{k=1}^K \sum_{(\tilde{s},\tilde{a}) \in \mathcal{D}_{h,i}} (g-g')^\top\Lambda_{h,i}^{-1/2} \phi_i(\tilde{s},\tilde{a})\big({q}_{h,i}^k(\tilde{s},\tilde{a}) - Q_{h,i}^{\pi_{h,-i}^k,\tilde{V}_{h+1,i}}(\tilde{s},\tilde{a}) \big)\\
  \lesssim& \lVert g - g' \rVert_{\mathcal{H}} H \sqrt{Kd \sum_{k=1}^K\sum_{(\tilde{s},\tilde{a}) \in \mathcal{D}_{h,i}} \phi_i(\tilde{s},\tilde{a})\Lambda_{h,i}^{-1} \phi_i(\tilde{s},\tilde{a})   }\\
  \lesssim& \lVert g - g' \rVert_{\mathcal{H}} HK\sqrt{d}. 
\end{align*}
That implies for any $g\in \mathbb{S}_{\mathcal{H}},$ there exists some $g_0 \in \mathcal{N}_\epsilon$ so that
\begin{align*}
  &\sum_{k=1}^K \sum_{(\tilde{s},\tilde{a}) \in \mathcal{D}_{h,i}} g^\top\Lambda_{h,i}^{-1/2} \phi_i(\tilde{s},\tilde{a})\big({q}_{h,i}^k(\tilde{s},\tilde{a}) - Q_{h,i}^{\pi_{h,-i}^k,\tilde{V}_{h+1,i}}(\tilde{s},\tilde{a}) \big)\\
  =& \sum_{k=1}^K \sum_{(\tilde{s},\tilde{a}) \in \mathcal{D}_{h,i}} g_0^\top\Lambda_{h,i}^{-1/2} \phi_i(\tilde{s},\tilde{a})\big({q}_{h,i}^k(\tilde{s},\tilde{a}) - Q_{h,i}^{\pi_{h,-i}^k,\tilde{V}_{h+1,i}}(\tilde{s},\tilde{a}) \big) +  O(\epsilon HK\sqrt{d}).
\end{align*}
Thus setting $\epsilon = \epsilon_0:= \frac{1}{\sqrt{K}}$ and taking union bound over $\mathcal{N}_\epsilon$ leads to with probability at least $1-\delta$, \begin{align*}
  \sup_{v\in \mathbb{S}^{d-1}} & \sum_{k=1}^K \sum_{(\tilde{s},\tilde{a}) \in \mathcal{D}_{h,i}} v^\top\Lambda_{h,i}^{-1/2} \phi_i(\tilde{s},\tilde{a})\big({q}_{h,i}^k(\tilde{s},\tilde{a}) - Q_{h,i}^{\pi_{h,-i}^k,\tilde{V}_{h+1,i}}(\tilde{s},\tilde{a}) \big) \\
  \lesssim &  H \sqrt{K\big[\log(\lvert \mathcal{N}_{\epsilon_0} \rvert / \delta) + d  \big]}.  \\
  \lesssim & H\sqrt{K\big[d\log(1 / \delta) \big]}
\end{align*}

That leads to another bound of $J_1$: with probability at least $1-\delta,$
\begin{align}\label{eq-J1-bound-d}
  J_1 \lesssim H\sqrt{\tau\dfrac{d\log(1 / \delta)}{K}}
\end{align}

Combining \eqref{eq-J1-bound-d} and \eqref{eq-J1-bound-SA} together leads to the desired result for $s\in\mathcal{C}_h.$ 
\end{proof}